\documentclass{article}

\usepackage[final]{ARLET_2024}

\usepackage[utf8]{inputenc} 
\usepackage[T1]{fontenc}    
\usepackage{hyperref}       
\usepackage{url}            
\usepackage{booktabs}       
\usepackage{amsfonts}       
\usepackage{nicefrac}       
\usepackage{microtype}      
\usepackage{xcolor}  

\usepackage{xcolor} 
\usepackage{graphicx}
\usepackage{amsthm,amsmath,amssymb}
\usepackage{mathrsfs}
\newtheorem{theorem}{Theorem}
\newtheorem{lemma}{Lemma}

\newtheorem{prop}{Proposition}
\usepackage{stfloats} 
\usepackage{miniplot}
\usepackage{bbm}
\usepackage{algorithm}
\usepackage{algorithmic}
\usepackage{multicol}
\usepackage{wrapfig}
\usepackage{subcaption}
\newtheorem{myDef}{Definition}
\usepackage{braket}
\usepackage{setspace}
\usepackage{lipsum}
\usepackage{dsfont}

\title{How Does Return Distribution in Distributional Reinforcement Learning Help Optimization?}

\author{%
  Ke Sun, Bei Jiang, Linglong Kong\thanks{Corresponding Author} \\
  Department of Mathematical and Statistical Sciences\\
  University of Alberta\\
  Edmonton, Canada \\
  \texttt{\{ksun6,bei1,lkong\}@ualberta.ca} \\
}

\begin{document}

\maketitle

\begin{abstract}
  Distributional reinforcement learning, which focuses on learning the entire return distribution instead of only its expectation in standard RL, has demonstrated remarkable success in enhancing performance. Despite these advancements, our comprehension of how the return distribution within distributional RL still remains limited. In this study, we investigate the optimization advantages of distributional RL by utilizing its extra return distribution knowledge over classical RL within the Neural Fitted Z-Iteration~(Neural FZI) framework. To begin with, we demonstrate that the distribution loss of distributional RL has desirable smoothness characteristics and hence enjoys stable gradients, which is in line with its tendency to promote optimization stability. Furthermore, the acceleration effect of distributional RL is revealed by decomposing the return distribution. It shows that distributional RL can perform favorably if the return distribution approximation is appropriate, measured by the variance of gradient estimates in each environment. Rigorous experiments validate the stable optimization behaviors of distributional RL and its acceleration effects compared to classical RL. Our research findings illuminate how the return distribution in distributional RL algorithms helps the optimization.
\end{abstract}

\section{Introduction}

Distributional reinforcement learning~\cite{bellemare2017distributional, dabney2017distributional, dabney2018implicit, yang2019fully, nguyen2020distributional,luo2021distributional, sun2022distributional} characterizes the intrinsic randomness of returns within the framework of Reinforcement Learning~(RL). When the agent interacts with the environment, the intrinsic uncertainty of the environment seeps into the stochasticity of rewards the agent receives and the inherently chaotic state and action dynamics of physical interaction, increasing the difficulty of the RL algorithm design. Distributional RL aims to represent the entire distribution of returns to capture more intrinsic uncertainty of the environment and, therefore, to use these return distributions to evaluate and optimize the policy. This is in stark contrast to the classical RL that only focuses on the expectation of the return distributions, such as temporal-difference~(TD) learning~\cite{sutton2018reinforcement} and Q-learning~\cite{watkins1992q}.

Despite the remarkable empirical success of distributional RL, the illumination of its theoretical advantages still needs to be studied. A distributional regularization effect~\cite{sun2021interpreting} stemming from the additional return distribution knowledge has been characterized to explain the superiority of distributional RL over classical RL, but the benefit of the proposed regularization on the optimization of algorithms has not been further investigated. Such a gap inspires us to investigate the optimization impact of distributional RL by leveraging the full return distribution knowledge. However, existing literature~\cite{mei2020global, strehl2009reinforcement} that helps to analyze the optimization of RL learning may not apply to practical distributional RL algorithms as there still remains a gap between the theory and practice in RL.

In this paper, we study the optimization advantages of distributional RL over classical RL. Within the Neural FZI framework, our optimization analysis can not only sufficiently characterize offline distributional RL behaviors but also approximate the online setting. Within this framework, we study the uniform stability of distributional loss based on categorical parameterization. Owing to the smoothness properties of distributional loss, distributional RL algorithms tend to satisfy the uniform stability in the optimization process, thus enjoying stable gradient behaviors in the input space. In addition to the optimization stability, we also elaborate on the acceleration effect of distributional RL algorithms based on the return density decomposition technique proposed recently. Distributional RL can speed up the convergence and perform favorably if the return distribution is approximated appropriately, measured by the gradient estimates' variance. Empirical results corroborate that distributional RL possesses stable gradient behaviors and acceleration effects by suggesting smaller gradient norms concerning the states and model parameters. Our study opens up many exciting research pathways in this domain through the lens of optimization, paving the way for future investigations to reveal more advantages of distributional RL. Our contributions can be summarized as follows:

\begin{itemize}
	\item  We specifically study the optimization advantage of practical distributional RL algorithms with the general function approximators. Within the Neural FZI framework, we can analyze the optimization properties of distributional RL by establishing its connection with supervised learning.
	\item  We reveal the uniform stability of distributional RL thanks to the smoothness properties of distributional loss. By contrast, classical RL may not guarantee such a stable optimization property due to the sensitivity of the least squared loss.
	\item The acceleration effects of distributional RL have also been demonstrated through the return density decomposition. We show that distributional RL can speed up convergence if the parameterization error of the return distribution is appropriate.
\end{itemize}

\section{Related Work}\label{sec:relatedwork}

\noindent \textbf{Interpretation of distributional RL.} Interpreting the behavior difference between distributional and classical RL was initially studied using the coupled updates method in \cite{lyle2019comparative}. They conclude that both distributional and classical RL behave the same in the tabular and linear approximation settings and attribute the superiority of distributional RL to its non-linear approximation. However, the coupled methodology mainly investigated preserving the expectation of return distribution to measure the behavior differences, which rules out other factors, including the optimization effect due to the distributional loss~\cite{imani2018improving}. An implicit risk-sensitive entropy regularization was then revealed in distributional RL by \cite{sun2021interpreting}, without further analyzing its optimization benefits. Our work complements and extends their results through the lens of optimization.

\noindent \textbf{Convergence and Acceleration in RL.} Existing optimization analysis in RL is mainly based on the policy gradient framework. \cite{mei2020global} shows that the policy gradient with a softmax parameterization converges at a $\mathcal{O}(1/t)$ rate, which significantly expands the existing asymptotic convergence results. Entropy regularization~\cite{haarnoja2017reinforcement,haarnoja2018soft} has gained increasing attention and \cite{ahmed2019understanding} provides a fine-grained understanding of the impact of entropy on policy optimization and emphasizes that any strategy, such as entropy regularization, can only affect learning in one of two ways: either it reduces the noise in the gradient estimates or it changes the optimization landscape. The seemingly applicable analysis framework on value-based RL is PAC-MDP~\cite{strehl2009reinforcement}, which effectively analyzes the convergence of typical RL algorithms in the tabular setting. However, it is unclear whether this analysis applies to practical distributional RL algorithms. By contrast, our optimization is within a more interpretable Neural FZI framework and focuses on accelerating the distributional RL algorithm.


\noindent \textbf{Stable Optimization.} Stable optimization is one of the crucial properties for RL algorithms, and common strategies include Batch Normalization~\cite{santurkar2018does},  Spectral Normalization~\cite{miyato2018spectral}, gradient penalty~\cite{gulrajani2017improved}. In RL, stable optimization techniques~\cite{gogianu2021spectral, li2021functional} also benefit the training and the final performance. By contrast, we show that (categorical) distributional RL naturally enjoys stable optimization compared with classical RL.

\section{Preliminary Knowledge}\label{sec:preliminary}

\noindent \textbf{Classical RL.} In a standard RL setting, the interaction between an agent and the environment is modeled as a Markov Decision Process~(MDP) ($\mathcal{S}, \mathcal{A}, R, P, \gamma$), where $\mathcal{S}$ and $\mathcal{A}$ denote state and action spaces. $P$ is the transition kernel dynamics, $R$ is the reward measure and $\gamma \in (0,1)$ is the discount factor. For a fixed policy $\pi$, the return, $Z^{\pi}=\sum_{t=0}^{\infty} \gamma^t R_t$, is a random variable representing the sum of discounted rewards
observed along one trajectory of states while following the policy $\pi$. Classical RL focuses on the value and action-value functions, the expectation of returns $Z^{\pi}$. The action-value function $Q^\pi(s, a)$ is defined as $Q^{\pi}(s, a)=\mathbb{E}\left[Z^{\pi}(s, a)\right]=\mathbb{E}\left[\sum_{t=0}^{\infty} \gamma^t R\left(s_t, a_t\right) \right]$, where $s_0=s$, $a_0=a$, $s_{t+1}\sim P(\cdot|s_t, a_t)$, and $a_t \sim \pi(\cdot|s_t)$. 

\noindent \textbf{Distributional RL.} Distributional RL, on the other hand, focuses on the action-state return distribution, the entire distribution of $Z^{\pi}(s, a)$ rather than only its expectation, i.e., $Q^\pi(s, a)$. Leveraging knowledge on the entire return distribution can better capture the uncertainty of returns and thus can be advantageous to explore the intrinsic uncertainty of the environment~\cite{dabney2018implicit,mavrin2019distributional}. Therefore, the scalar-based classical Bellman update is extended to the distributional Bellman update, which allows a flurry of distributional RL algorithms.

\noindent \textbf{Categorical Distributional RL~(CDRL).} As the first successful distributional RL family, CDRL~\cite{bellemare2017distributional} approximates the action-state return distribution $\eta$ by a categorical distribution $\hat{\eta}=\sum_{i=1}^{k}f_i \delta_{l_i}$ where $l_1, l_2, ..., l_k$ is a set of fixed supports and $\{f_i\}_{i=1}^k$ are learnable probabilities, normally parameterized by a neural network. A projection is also introduced to have the joint support with newly distributed target probabilities, equipped by a KL divergence to compute the distribution distance between the current and target return distribution within each Bellman update. In practice, C51~\cite{bellemare2017distributional}, an instance of CDRL with $k=51$, performs favorably in various environments.

\section{Optimization Analysis}\label{sec:optimization}


Under Neural FZI established in Section~\ref{sec:neuralFZI}, we analyze two optimization aspects of distributional RL based on the categorical parameterization, including the stable optimization from the loss function in Section~\ref{sec:stability}, and its acceleration effect determined by the gradient estimate variance in Section~\ref{sec:acceleration}.


\subsection{Optimization Analysis for Distributional RL within Neural Fitted Z-Iteration}\label{sec:neuralFZI}

In classical RL, \textit{Neural Fitted Q-Iteration}~(Neural FQI)~\cite{fan2020theoretical,riedmiller2005neural} provides a statistical interpretation of DQN~\cite{mnih2015human}, capturing its two key features, i.e., the leverage of target network and experience replay:
\begin{equation}\begin{aligned}\label{eq:Neural_Q_fitting}
		Q_\theta^{k+1}=\underset{Q_{\theta}}{\operatorname{argmin}} \frac{1}{n} \sum_{i=1}^{n}\left[y_{i}-Q_\theta^k\left(s_{i}, a_{i}\right)\right]^{2},
\end{aligned}\end{equation}
where the target $y_{i}=r(s_i, a_i)+\gamma \max _{a \in \mathcal{A}} Q^k_{\theta^*} \left(s_{i}^{\prime}, a\right)$ is fixed within every $T_{\text{target}}$ steps to update target network $Q_{\theta^*}$ by letting $\theta^*=\theta$. The experience buffer induces independent samples $\left\{\left(s_{i}, a_{i}, r_{i}, s_{i}^{\prime}\right)\right\}_{i \in[n]}$ and ideally without the optimization and TD approximation errors, Neural FQI is exactly the update under Bellman optimality operator~\cite{fan2020theoretical}. Similarly, \cite{sun2021interpreting,ma2021conservative} proposed \textit{Neural Fitted Z-Iteration}~(Neural FZI), a distributional version of Neural FQI based on the parameterization of $Z_\theta$ to characterize distributional RL:
\begin{equation}
	\begin{aligned}\label{eq:Neural_Z_fitting}
		Z_\theta^{k+1}=\underset{Z_{\theta}}{\operatorname{argmin}} \frac{1}{n} \sum_{i=1}^{n} d_p (Y_{i}, Z_\theta^k\left(s_{i}, a_{i}\right)),
	\end{aligned}
\end{equation}
where the target $Y_{i}=R(s_i, a_i)+\gamma Z^k_{\theta^*} \left(s_{i}^{\prime}, \pi_Z(s_i^\prime)\right)$ is a random variable, whose distribution is also fixed within every $T_{\text{target}}$ steps. The target follows a greedy policy rule, where  $\pi_Z(s^\prime_i)= \operatorname{argmax}_{a^\prime} \mathbb{E}\left[Z_{\theta^*}^{k}(s_i^\prime, a^\prime)\right]$ and $d_p$ is the choice of distribution distance. 

\noindent \textbf{Approximate Supervised Learning within Neural FZI to Allow the Optimization Analysis.} Previous optimization analysis focuses on either policy gradient methods~\cite{mei2020global,agarwal2020optimality} or the sample complexity in the tabular setting~\cite{strehl2009reinforcement}. However, there remains some gap between the theory and the practical neural network parameterized RL algorithm, and the previous results may not be directly attainable for the optimization analysis of distributional RL. By contrast, Neural FZI simplifies the optimization problem in deep RL into an approximate iterative supervised learning on a local fixed offline dataset by leveraging experience buffer and target networks, allowing richer optimization analysis. It sufficiently characterizes the offline behaviors of practical distributional RL algorithms and can also approximate online algorithms. In particular, Neural FZI does not consider the exploration; the data distribution shift caused by exploration from an $\epsilon$-greedy policy can be negligible in the online setting, \textit{when the replay memory is sufficiently large or considering the short period.} Thus, the optimization in each phase of Neural FZI can be approximately viewed as supervised learning in contrast to PAC-MDP analysis~\cite{strehl2009reinforcement} that explicitly involves the impact of exploration.

\noindent \textbf{Two Key Factors.} The Neural FZI framework offers new insights to analyze the optimization benefits for practical distributional RL algorithms, within which there are mainly two crucial components. 

\begin{itemize}
	\item \textbf{Factor 1: the choice of $d_p$}. On the one hand, $d_p$ determines the convergence rate of distributional Bellman update, i.e., the speed of outer iterations in Neural FZI. For instance, distributional Bellman operator under Crámer distance is $\sqrt{\gamma}$-contractive~\cite{bellemare2017cramer},  $\gamma$-contractive under  Wasserstein distance~\cite{bellemare2017distributional}. Moreover, $d_p$ also largely affects the continuous optimization problem concerning parameters $\theta$ in $Z_\theta$ within each iteration of Neural FZI. 
	
	\item \textbf{Factor 2: the parameterization of $Z_\theta$}. Given the same $d_p$, a more informative parameterization can approximate the true return distribution more reasonably, promoting the optimization within each phase of Neural FZI. For example, with a more expressiveness power on quantile functions, IQN~\cite{dabney2018implicit} outperforms QR-DQN~\cite{dabney2017distributional} on a wider range of environments.
\end{itemize}

\noindent \textbf{Remark.} We mainly attribute the optimization benefit of distributional RL to the choice of distributional loss  $d_p$ in Neural FZI  relative to the least squared loss in Neural FQI  based on the same categorical parameterization on $Z_\theta$, despite the different convergence rates under them.

\noindent \textbf{Categorical Pameterization Equipped with KL Divergence.} To allow for theoretical analysis, we resort to the histogram function~\cite{wasserman2006all, imani2018improving} as the density estimator of $Z_\theta$, a continuous version of categorical parameterization with their equivalent proof provided in \cite{sun2021interpreting}. After incorporating the projection to redistribute probabilities of target return distribution by the neighboring smoothing proposed in CDRL, the target, and current histogram function estimators inherit the joint supports, based on which we apply KL divergence as $d_p$. In particular, we denote the histogram density estimator as $f^{s, a}$ with $k$ uniform partitions on the support, denote $\mathbf{x}(s)$ as the state feature on each state $s$. We let the support of $Z(s, a)$ be uniformly partitioned into $k$ bins. The output dimension of $f^{s,\cdot}$ can be $|\mathcal{A}| \times k$, where we use the index $a$ to focus on the function $f^{s, a}$. Hence, the function $f^{s, a}: \mathcal{X} \rightarrow[0,1]^{k}$ provides a $k$-dimensional vector $f^{s, a}(\mathbf{x}(s))$ of the coefficients, indicating the probability that the target is in this bin given the state feature $\mathbf{x}(s)$ and action $a$. Next, we use \textit{softmax} based on the linear approximation $\mathbf{x}(s)^{\top} \theta_{i}$ to express $f^{s, a}$, i.e., $f_{i}^{s, a, \theta}(\mathbf{x}(s))=\exp \left(\mathbf{x}(s)^{\top} \theta_{i}\right) / \sum_{j=1}^{k} \exp \left(\mathbf{x}(s)^{\top} \theta_{j}\right)$. For simplicity, we use $f_{i}^{\theta}(\mathbf{x}(s))$ to replace $f_{i}^{s, a, \theta}(\mathbf{x}(s))$. 

\noindent \textbf{Categorical Distributional Loss.} Note that the form of $f^{s, a}$ is similar to that in Softmax policy gradient optimization~\cite{mei2020global,sutton2018reinforcement}, but we focus on the value-based RL rather than the policy gradient RL. Our prediction probability $f_i^{s, a}$ is redefined as the probability in the $i$-th bin over the support of $Z(s, a)$, thus eventually serving as a density function. While the linear approximator is limited, this is the setting where, so far, the cleanest results can be firstly achieved, and understanding this setting is necessary for the first step towards bigger problems of understanding distributional RL algorithms. Under this categorical parameterization with KL divergence, the distributional objective function $\mathcal{L}_\theta(s, a)$ for the continuous optimization in each phase of Neural FZI~(Eq.~\ref{eq:Neural_Z_fitting}) can be expressed as:
\begin{equation}\begin{aligned}\label{eq:histogram} 
		\mathcal{L}_\theta(s, a)  &= -\sum_{i=1}^{k} \int_{z_{i}}^{z_{i}+w_{i}} p^{s, a}(y) \log \frac{f_{i}^\theta(\mathbf{x}(s))}{w_{i}} d y \propto -\sum_{i=1}^{k} p^{s, a}_{i} \log f_{i}^\theta(\mathbf{x}(s)),
\end{aligned}\end{equation} 
where $\theta=\{\theta_1, ..., \theta_{k}\}$ and $p^{s, a}_i$ is the probability in the $i$-th bin of the true density function $p^{s, a}(x)$ for $Z(s, a)$ defined in Eq.~\ref{eq:decomposition}. $w_i$ is the width for the $i$-th bin $(z_i, z_{i+1}]$. The derivation of the categorical distributional loss under the categorical parameterization is given in Appendix~\ref{appendix:histogram}. 

\subsection{Stable Optimization Analysis under Uniform Stability}\label{sec:stability}

\noindent \textbf{Optimization Properties.} Our stable optimization conclusions are based on the smoothness properties of categorical distributional loss in Eq.~\ref{eq:histogram}. A similar histogram loss was also analyzed in \cite{imani2018improving} along with a local Lipschitz constant analysis. By contrast, in Proposition~\ref{prop:lipschitz}, we extend their optimization results and further establish its connection with distributional RL.

\begin{prop}\label{prop:lipschitz} (Properties of Categorical Distributional Loss)
	Assume the state features $\Vert \mathbf{x}(s) \Vert_2 \leq l$ for each state $s$, then $\mathcal{L}_\theta$ is $kl$-Lipschitz continuous, $kl^2$-smooth and convex w.r.t. the parameter $\theta$.
\end{prop}

Please refer to Appendix~\ref{appendix:lemma_lipschitz} for the proof. The smoothness properties of categorical distributional loss  $d_p$ are the foundation for the stable optimization of distributional RL. In stark contrast, classical RL optimizes a least squared loss function~\cite{sutton2018reinforcement} in Neural FQI. It is known that the least squared estimator has no bounded Lipschitz constant in general and is only $\lambda_\text{max}$-smooth, where $\lambda_\text{max}$ is the largest singular value of the data matrix. Specifically, we have $\Vert \nabla_\theta \mathcal{L}_\theta \Vert \le kl$ for the categorical distributional loss in distributional RL. By contrast, the gradient norm in classical RL is $|y_i - Q_\theta^k(s, a)|\Vert \mathbf{x}(s) \Vert$, where  $Q_\theta^k(s, a)=\sum_{i=1}^{k} (z_i+z_{i+1})f_{i}^\theta(\mathbf{x}(s)) / 2 w_{i}$ under the same categorical parameterization for a fair comparison. Clearly, $Q_\theta^k(s, a)$ can be sufficiently large if the support $[z_0, z_k]$ is specified to be large, which is common in environments with a high level of expected returns~\cite{bellemare2017distributional}. As such, $|y_i - Q_\theta^k(s, a)|$ can vary significantly larger than $k$ and classical RL with the potentially larger upper bound of gradient norms is prone to the instability optimization issue.

\noindent \textbf{Uniform Stability of Distributional RL.} As an application of stable analysis in \cite{hardt2016train}, we next show that distributional RL loss can naturally induce a uniform stability property under the desirable smoothness properties in Proposition~\ref{prop:lipschitz}, while classical RL can not. We first recap the definition of uniform stability for an algorithm while running \textit{Stochastic Gradient Descent}~(SGD) in Definition~\ref{def:stability}.

\begin{myDef}\label{def:stability}(Uniform Stability)~\cite{hardt2016train}
	Consider a loss function $g_w(e)$ parameterized by $w$ encountered on the example $e$, a randomized algorithm $\mathcal{M}$ is uniformly stable if for all data sets $\mathcal{D}, \mathcal{D}^\prime$ such that $\mathcal{D}, \mathcal{D}^\prime$ differ in at most one example, we have 
	\begin{equation}\begin{aligned}\label{eq:uniform_stable}
			\sup_{e} \mathbb{E}_{\mathcal{M}}\left[g_{\mathcal{M}(\mathcal{D}) }(e)-g_{\mathcal{M}\left(\mathcal{D}^{\prime}\right) }\left(e\right)\right] \leq \epsilon_{\text {stab }}.
	\end{aligned}\end{equation}
\end{myDef}

\noindent \textbf{Remark: Rationale of Uniform Stability Analysis.} One may be concerned whether the uniform stability analysis is applicable to the RL setting with a gradually varying experience replay buffer. Thanks to the Neural FZI framework, it can be viewed as an approximate supervised learning on a nearly fixed offline dataset $\mathcal{D}$ with each iteration of Neural FZI, as the experiment replay allows nearly independent sampling on a fixed data distribution in a short period when the reply memory is large enough~\cite{fan2020theoretical}.  As such, the loss difference by varying the dataset for at most one sample can serve as a surrogate to measure the uniform stability for an algorithm in each phase of Neural FZI.

\begin{theorem}\label{theorem:lipschitz} (Uniform Stability for Distributional RL) Suppose that we run SGD under $\mathcal{L}_\theta$ in Eq.~\ref{eq:histogram} with step sizes $\lambda_t \le 2 / kl^2$ for $T$ steps. Assume $\Vert \mathbf{x}(s) \Vert \leq l$ for each state $s$ and action $a$, then we have $\mathcal{L}_\theta$ satisfies the uniform stability in Definition~\ref{def:stability} with $\epsilon_{\text {stab }} \leq \frac{4kT}{n}$, i.e., 
	\begin{eqnarray}
		\begin{aligned}
			\mathbb{E}\left|\mathcal{L}_{\theta_T}(s, a) - \mathcal{L}_{\theta_T^\prime}(s, a)\right| \leq \frac{4kT}{n},
		\end{aligned}
	\end{eqnarray}
	where $\theta_T$ and $\theta_T^\prime$ are the minimizers after $T$ steps under the dataset $\mathcal{D}$ and $\mathcal{D}^\prime$, respectively.
\end{theorem}

Please refer to the proof of Theorem~\ref{theorem:lipschitz} in Appendix~\ref{appendix:lipschitz}. Theorem~\ref{theorem:lipschitz} shows that while running SGD to solve the categorical distributional loss within each Neural FZI, the continuous optimization process in each iteration is $\epsilon_{\text{stab}}$-uniformly stable with the stability errors shrinking at the rate of $O(n^{-1})$. The stable optimization has multiple advantages, including  $\epsilon_{\text {stab }}$-bounded generalization gap, a desirable local minimum in deep learning optimization literature~\cite{hardt2016train}, and improvement in performance in RL~\cite{bjorck2021towards,li2021functional}. By contrast, classical RL may not yield thestable optimization property without these smooth properties. For example, $\lambda_\text{max}$-smooth may be of less help for the optimization given a bad conditional number of the design matrix where $\lambda_\text{max}$ could be sufficiently large. Empirically, we validate the stable gradient behaviors, with smaller gradient norms in the input space, of CDRL compared with classical RL, and similar results are also observed in Quantile Regression distributional RL in Section~\ref{sec:experiments}. 

\noindent \textbf{Remark: Limitations.} The potential optimization instability for classical RL can be used to partially explain its inferiority to distributional RL in most environments, although it may not explain why distributional RL could not perform favorably in certain games~\cite{ceron2021revisiting}. We leave the comprehensive explanation as future works.

\noindent \textbf{Remark: Non-linear Categorical Parameterization.} Although the stability above optimization conclusions are established on the linear categorical parameterization on  $Z^\pi$, similar conclusions with a non-linear categorical parameterization can be naturally expected by non-convex optimization techniques proposed in \cite{hardt2016train}. We empirically validate our theoretical conclusions by directly applying practical neural network parameterized distributional RL algorithms.

\subsection{Acceleration Effect of distributional RL}\label{sec:acceleration}

To characterize the acceleration effect of distributional RL, we additionally leverage the recently proposed \textit{return density function decomposition}~\cite{sun2021interpreting}.

\noindent \textbf{Return Density Function Decomposition.} To decompose the optimization impact of return distribution into its expectation and the remaining distribution part, we apply the return density function decomposition to decompose the target histogram density function $p^{s, a}$. This decomposition was successfully applied to derive the distributional regularization effect of distributional RL and was rigorously justified in \cite{sun2021interpreting}.  Based on the categorical parameterization, we denote $\Delta_{E}$ as the interval that $\mathbb{E}\left[Z\pi(s, a)\right]$ falls into, i.e., $\mathbb{E}\left[Z\pi(s, a)\right] \in \Delta_{E}$, and the categorical parameterized $p^{s, a}(x) = \sum_{i=1}^N f_i^\theta \mathds{1}(x\in \Delta_i) /\Delta$ can be decomposed as
\begin{equation}\begin{aligned}\label{eq:decomposition}
		p^{s, a}(x) & = (1-\epsilon) p_E^{s, a} + \epsilon \mu^{s, a} = (1-\epsilon) \mathds{1}(x\in \Delta_E) /\Delta + \epsilon \sum_{i=1}^{N}p^\mu_i \mathds{1}(x\in \Delta_i) /\Delta,
\end{aligned}\end{equation}
where $p^{s, a}$ is decomposed into a single-bin histogram density $ p_E^{s, a} $ and an induced one $\widehat{\mu}^{s,a}$ with each bin probability $p_i^\mu$ determined by $f_i^\theta$. The pre-specified $\epsilon$ measures the impact of remaining distribution $\widehat{\mu}^{s,a}$ \textit{independent of} its expectation $\mathbb{E}\left[Z^\pi(s, a)\right]$. It shows optimizing the first term in Eq.~\ref{eq:decomposition} is equivalent to the classical RL loss in Neural FQI~\cite{sun2021interpreting}, which we provide the proof in Appendix~\ref{appendix:equivalence} for completeness. Therefore, this decomposition allows us to conduct the acceleration effect of distributional RL loss as opposed to classical RL.

\noindent \textbf{Measuring the Variance of Gradient Estimates.} Within Neural FZI, our goal is to minimize $\frac{1}{n}\sum_{i=1}^{n}\mathcal{L}_\theta(s_i, a_i)$. 
We rewrite $\mathcal{L}_\theta(s, a)$ as $\mathcal{L}_\theta(g^{s, a}, f^{s, a}_\theta)$, where the target density function $g^{s, a}$ can be $p^{s, a}$, $\mu^{s, a}$ or $p_E^{s, a}$, and $f^{s, a, \theta}$ is rewritten as $f^{s, a}_\theta$ for conciseness. We denote $G^k(\theta) = \mathbb{E}\left[\mathcal{L}_\theta( p_E^{s, a} , f_\theta^{s, a})\right]$ and use $G(\theta)$ for $G^k(\theta)$ for simplicity. Based on Proposition~\ref{prop:lipschitz}  in Section~\ref{sec:stability}, the appealing optimization properties concerning the parameter $\theta$ in $f_\theta$ still hold for $G(\theta)$. Although $p^{s, a}_E$ is a single-bin density without non-zero joint support as $f^{s, a}_\theta$, thanks to the leverage of target networks, the KL-based $\mathcal{L}_\theta$ would degrade to the cross-entropy loss, on which $\mathcal{L}_\theta$ is still well-defined. As the KL divergence has unbiased gradient estimates, we let the variance of its stochastic gradient over the expectation-related term $p^{s, a}_E$  be bounded, i.e.,
\begin{equation}\begin{aligned}
		\mathbb{E}_{(s, a)\sim \rho^\pi}\left[\|\nabla \mathcal{L}_\theta(p^{s, a}_E, f_\theta^{s, a}))-\nabla G(\theta)\|^{2}\right]=\sigma^{2}.
\end{aligned}\end{equation}
Next, following the similar label smoothing analysis in \cite{xu2020towards}, we further characterize the approximation degree of $f^{s, a}_\theta$ to the target return distribution $\mu^{s, a}$ by measuring its variance as $\kappa \sigma^2$:
\begin{equation}\begin{aligned}\label{eq:acceleration_kappa}
		\mathbb{E}_{(s, a)\sim \rho^\pi}\left[\|\nabla \mathcal{L}_\theta(\mu^{s, a}, f_\theta^{s, a}))-\nabla G(\theta)\|^{2}\right]=\hat{\sigma}^2:=\kappa \sigma^{2}.
\end{aligned}\end{equation}
Notably, $\kappa$ can be used to measure the approximation error between  $f_\theta^{s, a}$ and $\mu^{s, a}$ and we do not assume $\hat{\sigma}^2$ to be bounded as $\kappa$ can be arbitrarily large. This expression $\kappa \sigma^2$ for $\hat{\sigma}^2$ allows us to utilize $\kappa$ to characterize different acceleration effects for distributional RL given different $\kappa$. Concretely, a favorable approximation of $f_\theta^{s, a}$ to $\mu^{s, a}$, which coincides with the role of the $Z_\theta$ parameterization, will lead to a small $\kappa$, contributing to the acceleration effect of distributional RL as shown in Theorem~\ref{theorem:acceleration}. 
\begin{prop}\label{prop:acceleration}
	Based on the return density decomposition in Eq.~\ref{eq:decomposition}, and Eq.~\ref{eq:acceleration_kappa}, we have:
	\begin{equation}\begin{aligned}\label{eq:acceleration_kappa_complete}
			& \mathbb{E}_{(s, a)\sim \rho^\pi}\left[\|\nabla \mathcal{L}_\theta(p^{s, a}, f_\theta^{s, a}))-\nabla G(\theta)\|^{2}\right] \le (1-\epsilon)^2\sigma^{2} + \epsilon^2 \kappa \sigma^{2}.
	\end{aligned}\end{equation}
\end{prop}
Proposition~\ref{prop:acceleration} reveals the upper bound of gradient estimate variance for the whole target density function $p^{s, a}$, with proof in Appendix~\ref{appendix:acceleration_lemma}. Before comparing the sample complexity in optimizing both classical and distributional RL, we define the first-order $\tau$-stationary point.

\begin{myDef}\label{definition:acceleration} (First-order $\tau$-Stationary Point)
	When $\min_\theta G(\theta)$, the updated parameters $\mathbb{\theta}_T$ after $T$ steps is a first-order $\tau$-stationary point if $\Vert \nabla_\theta G(\theta_T) \Vert \le \tau$.
\end{myDef}

Based on Definition~\ref{definition:acceleration}, we formally characterize the acceleration effects for distributional RL in Theorem~\ref{theorem:acceleration} that depends upon approximation errors between $\mu^{s, a}$ and $f^{s, a}_\theta$ measured by $\kappa$.

\begin{theorem}\label{theorem:acceleration} (Sample Complexity and Acceleration Effects of Distributional RL) While running SGD to minimize $\mathcal{L}_\theta$ in Eq.~\ref{eq:decomposition} within Neural FZI, we assume the step size $\lambda \leq \frac{1}{kl^2} \min\{1, \frac{\tau^2}{2\sigma^2}\}$, $\epsilon=1/(1+\kappa)$, and the sample is uniformly drawn from $T$ samples. Denote $G(\theta_0)$ as initialization.
	
	(1) (\textbf{Classical RL}) The sample complexity $T = \frac{4 G(\theta_0)}{\lambda \tau^2} = O(\frac{1}{\tau^4})$  when minimizing $\mathcal{L}_\theta(p^{s, a}_E, f_\theta^{s, a})$, such that $\mathcal{L}_\theta$ converges to a $\tau$-stationary point in expectation.
	
	(2) (\textbf{Distributional RL}) The sample complexity $T = O(\frac{1}{\tau^2})$  when minimizing $\mathcal{L}_\theta(p^{s, a}, f_\theta^{s, a})$, such that $\mathcal{L}_\theta$ converges to a $\max\{\tau, 2\sigma \kappa\}$-stationary point in expectation.

\end{theorem}

The proof is provided in Appendix~\ref{appendix:acceleration_theorm}. Theorem~\ref{theorem:acceleration} is inspired by the intuitive connection between the return distribution in distributional RL and the label distribution in label smoothing~\cite{xu2020towards}.

\begin{figure*}[b!]
	\centering
	\includegraphics[width=1.0\textwidth,trim=0 0 0 50,clip]{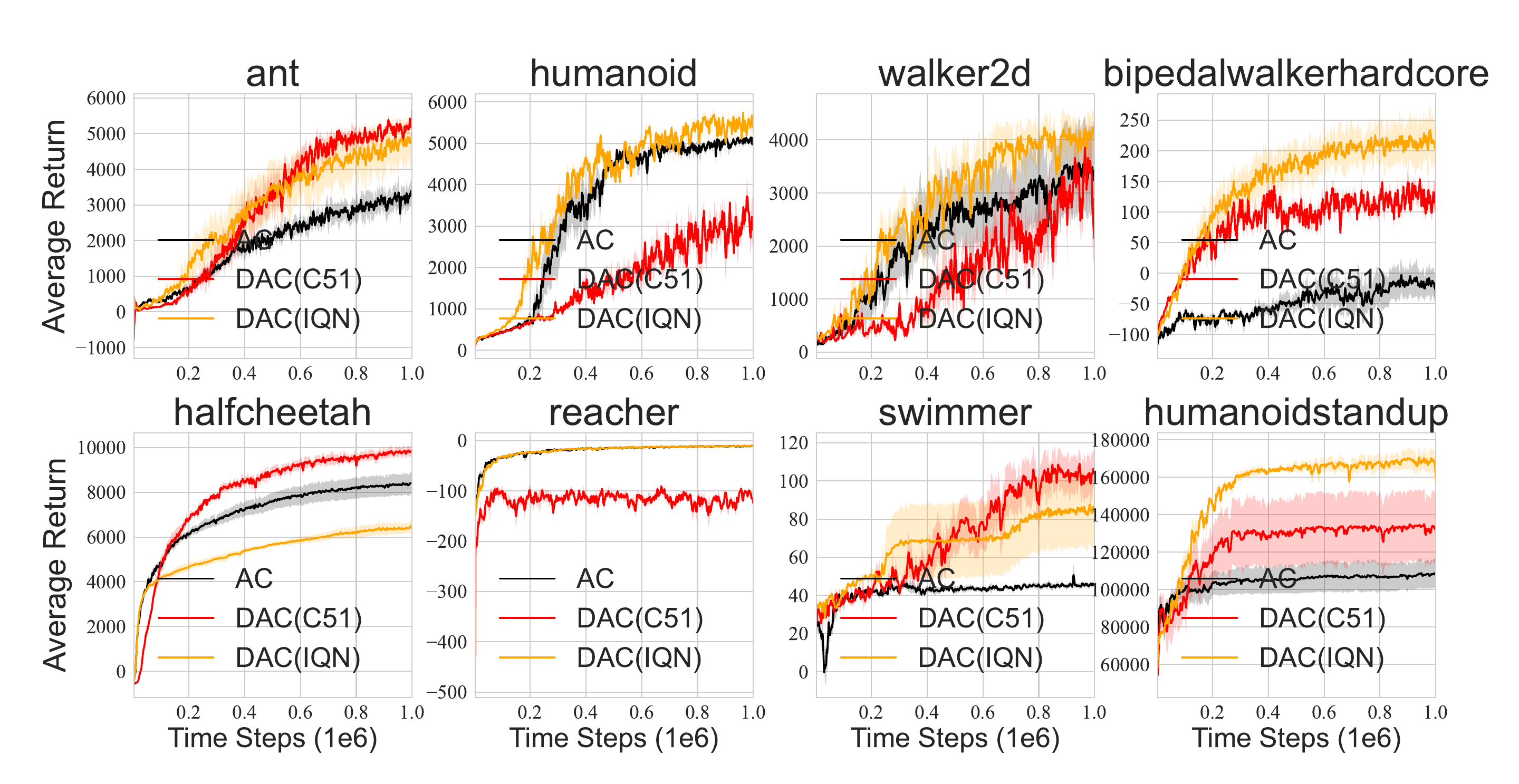}
	\caption{\textbf{Performance.} Learning curve of AC, DAC~(C51), and DAC~(IQN) over five seeds with smooth size five across eight MuJoCo games.}
	\label{fig:performance}
\end{figure*}

\noindent \textbf{Interpretation of Theorem~\ref{theorem:acceleration}}. Theorem~\ref{theorem:acceleration} demonstrates that optimizing the categorical distributional loss of distributional RL can speed up the convergence with the sample complexity from $O(\frac{1}{\tau^4})$ to $O(\frac{1}{\tau^2})$, if the distribution approximation error is favorable. In particular, when the agnostic $\kappa$ determined by the environment satisfies $2\kappa\sigma \leq \tau$, the distributional RL algorithm has an effective return distribution parameterization for $Z_\theta$ with a smaller approximation error between $f_\theta^{s, a}$ and $\mu^{s, a}$ ($p^{s, a}$). In this case, the acceleration effect of distributional RL over classical RL can be guaranteed. However, it is not vice versa. When $2\kappa\sigma > \tau$, it is unclear whether the required sample complexity for distributional RL is higher than classical RL, as classical RL will require a lower sample complexity than $O(\frac{1}{\tau^4})$ to achieve a $2\kappa\sigma$-stationary point in this case. These theoretical results also coincide with past empirical observations~\cite{dabney2017distributional, ceron2021revisiting}, where distributional RL algorithms outperform classical RL in most cases, but are inferior in certain environments. Based on our results in Theorem~\ref{theorem:acceleration}, we contend that these certain environments have much intrinsic uncertainty, the distribution parameterization error between $Z_\theta$ and the true return distribution under the distributional TD approximation is still too large~($\kappa > \frac{\tau}{2\sigma}$) to guarantee an acceleration effect as revealed in Theorem~\ref{theorem:acceleration}.

\noindent \textbf{Smaller Gradient Norms in the Weight Space.} The acceleration effect of distributional RL in Theorem~\ref{theorem:acceleration} also implies that distributional RL tends to have smaller gradient norms concerning parameters than classical RL at the same training step, according to the definition of Lipschitz constant in terms of the first-order stationary point.  The small gradient norms we analyze here are \textit{in the weight space}, commonly used and directly linked with the convergence rate analysis. In contrast, the uniform stability analyzed in Section~\ref{sec:stability} is defined on the bounded loss difference that is strongly correlated to the gradient norms \textit{in the input space}. Similar works include Spectral Normalization to stabilize the training of Generative Adversarial Networks~\cite{miyato2018spectral} and RL~\cite{gogianu2021spectral}, which normalizes the spectral norm of the weight matrix in each layer to lead to a one-valued Lipschitz constant concerning the input. We empirically demonstrate both of them in Section~\ref{sec:experiments}.

\begin{figure*}[b!]
	\centering
	\includegraphics[width=1.0\textwidth,trim=0 0 0 0,clip]{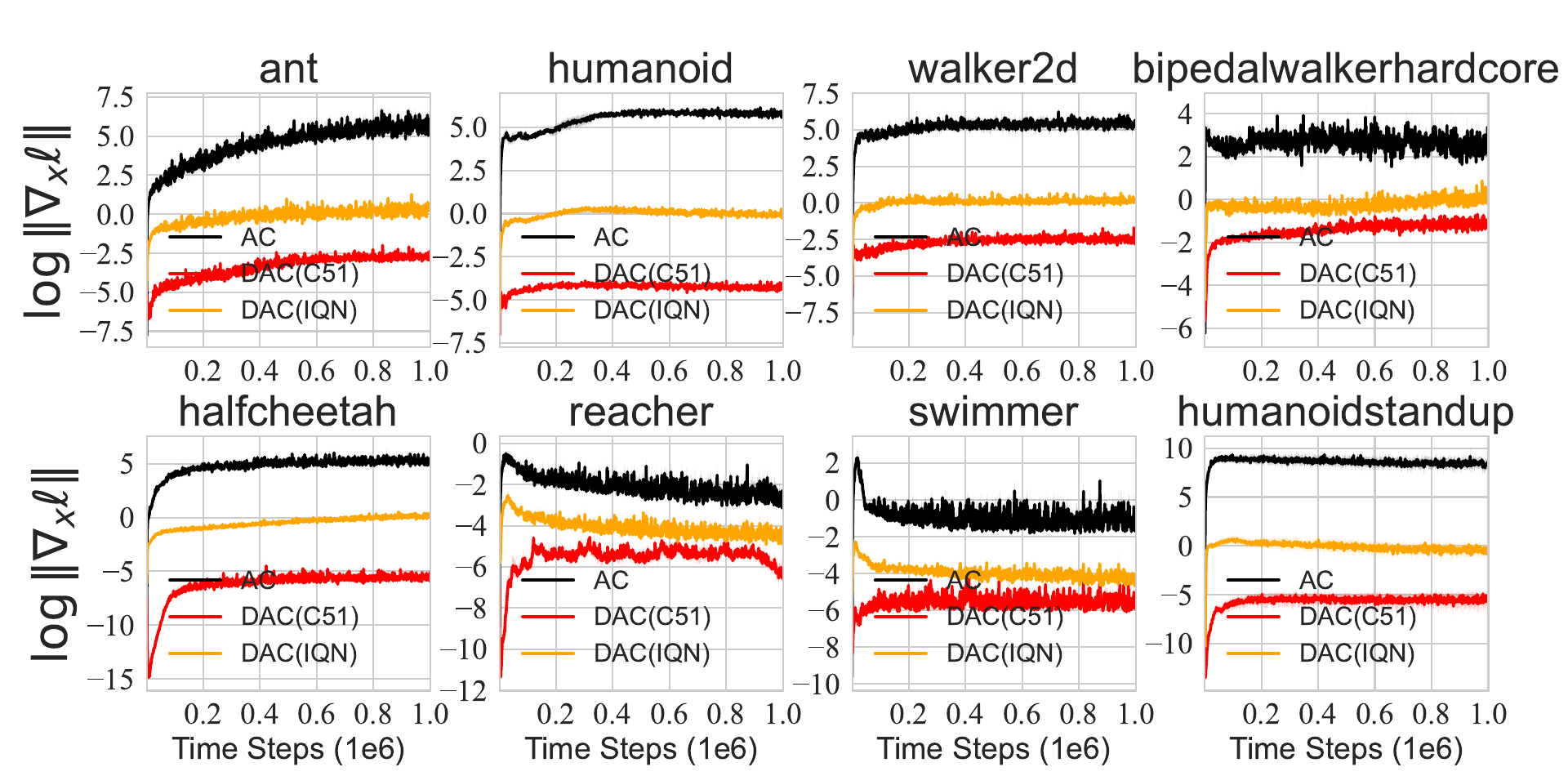}
	\caption{ \textbf{Uniform Stability.} The critic gradient norms in the logarithmic scale regarding \textbf{the state} during the training of AC, DAC~(C51), DAC~(IQN) over 5 seeds on eight MuJoCo environments.}
	\label{fig:optimization}
\end{figure*}

\section{Experiments}\label{sec:experiments}

Our experiments focus on the online distributional RL algorithms on continuous control Mujoco environments to demonstrate their stable gradient behaviors and acceleration effects.


\noindent \textbf{Implementation.} Our implementation is based Soft Actor Critic~(SAC)~\cite{haarnoja2018soft} and distributional Soft Actor Critic~\cite{ma2020dsac}. We eliminate the optimization impact of entropy regularization in these algorithm implementations, and thus, we denote the resulting algorithms as Actor Critic~(AC) and Distributional Actor Critic~(DAC) for conciseness. For DAC, we first perform a categorical parameterized C51 critic loss from the classical least-squared critic loss dubbed DAC~(C51), which coincides with our theoretical analysis in Sections~\ref{sec:stability} and \ref{sec:acceleration}. We further apply our experiments on Quantile Regression distributional RL, i.e., Implicit Quantile Network~(IQN), denoted as DAC~(IQN), to heuristically extend our conclusion in broader algorithm classes. More implementation details are provided in Appendix~\ref{appendix:implementation}. 

\subsection{Performance and Uniform Stability}

Figure~\ref{fig:performance} suggests both  DAC~(IQN) and DAC~(C51) excel at the classical RL counterpart, i.e., AC~(black lines), in most environments, which allows our further optimization analysis.


\begin{figure*}[t!]
	\centering
	\includegraphics[width=1.0\textwidth,trim=0 0 0 0,clip]{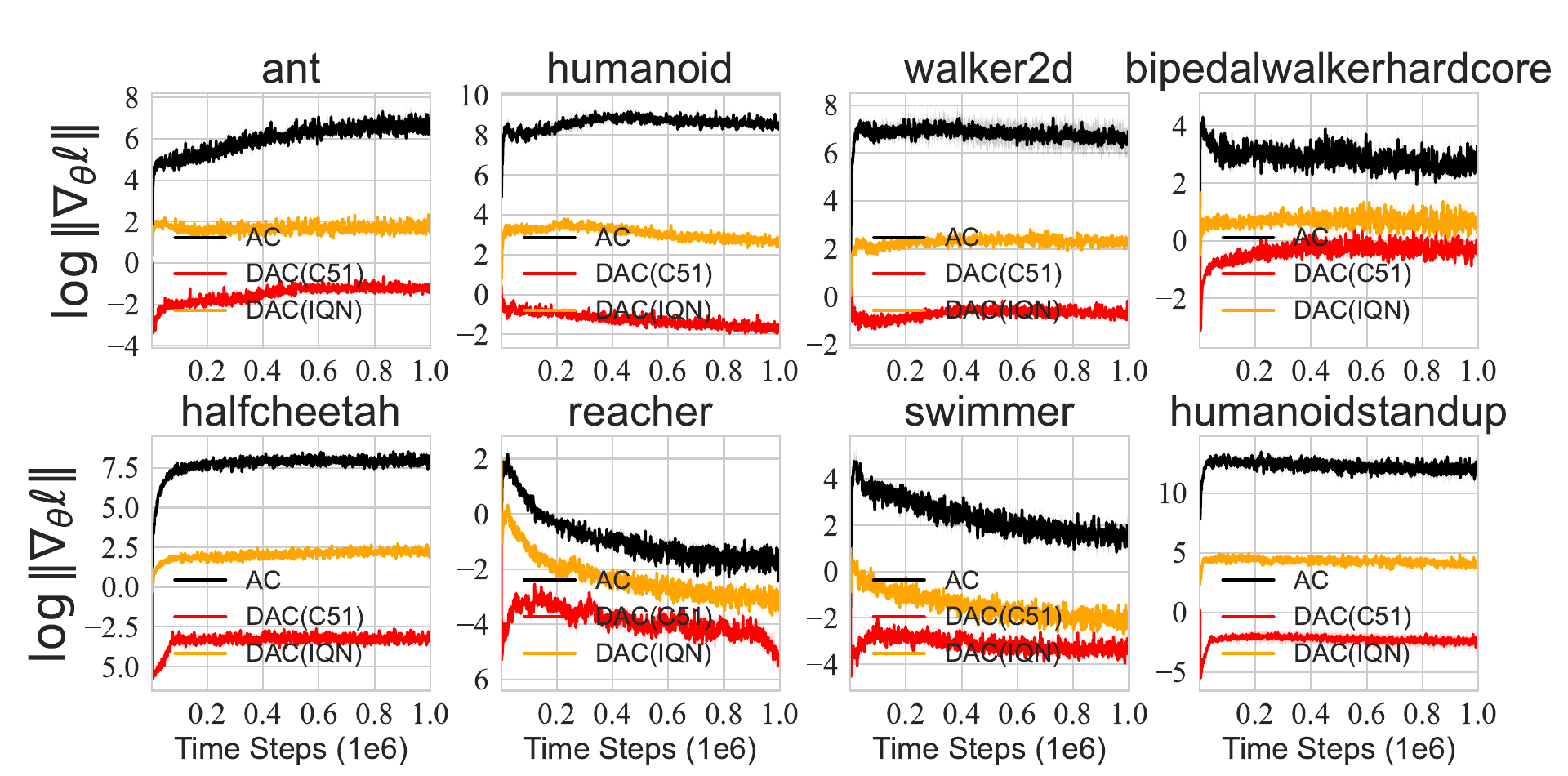}
	\caption{ \textbf{Acceleration Effect.} The critic gradient norms in the logarithmic scale regarding \textbf{network parameters} in the training of AC, DAC~(C51), DAC~(IQN) over 5 seeds on MuJoCo environments.}
	\label{fig:acceleration}
\end{figure*}

\noindent \textbf{Proxy: Gradient Norms in the Input Space.} We demonstrate the advantage of uniform optimization stability for distributional RL over classical RL. According to Theorem~\ref{theorem:lipschitz}, the stable optimization of distribution loss within Neural FZI is described as a bounded loss difference for a neighboring dataset regarding each state $s$ and action $a$. In other words, the error bound holds by taking the supreme over each state and action pair. To measure this algorithm stability, while far from perfect, we consider leveraging \textit{the average gradient norms concerning the state feature $\mathbf{x}(s)$} in the whole optimization process as the proxy. This is because the gradient magnitude in the input space could measure the sensitivity of the loss function regarding each state and action pair.

\noindent \textbf{Results.} Figure~\ref{fig:optimization} suggests that both DAC~(C51) and DAC~(IQN) entail a much smaller gradient norm magnitude than the classical AC~(black lines) across all environments, corroborating the advantage of uniform stability for distributional RL over classical RL analyzed in Theorem~\ref{theorem:lipschitz}. As analyzed in Section~\ref{sec:stability}, this result provides empirical evidence to interpret behaviors of distributional RL.

\subsection{Acceleration Effect of Distributional RL} 

\noindent \textbf{Proxy: Gradient Norms in the Weight Space}. Theorem~\ref{theorem:acceleration} implies that if the distribution parameterization is appropriate, distributional RL can speed up the convergence and thus can achieve better first-order stationary point, corresponding to smaller gradient norms given the time step in the learning process. To demonstrate it, we take the same step size for both DAC and AC, and evaluate the $\ell_2$-norms of gradients  \textit{concerning network parameters} of their critics. A direct comparison between vanilla AC and DAC algorithm is given in Figure~\ref{fig:acceleration}, despite the slight difference in the network architecture in the last layer. For an \textit{apple-to-apple comparison}, we keep the same DAC architecture while implementing a variant AC by optimizing the expectation of represented return distribution. We also find a similar result in  Appendix~\ref{appendix:exp_acceleration_kappa}.

\noindent \textbf{Results.} Figure~\ref{fig:acceleration} showcases that both DAC~(C51) and DAC~(IQN) have smaller gradient norms in terms of network parameters $\theta$ compared with AC in the whole optimization process. This result also validates that distributional RL loss tends to enjoy smoothness properties in Proposition~\ref{prop:lipschitz}. Moreover, it turns out that DAC~(IQN) tends to have smaller gradient norms than DAC~(C51). Given the fact that DAC~(IQN) outperforms DAC~(C51) in most environments in Figure~\ref{fig:performance}, we hypothesize that DAC~(IQN) may have a better acceleration effect than DAC~(C51), contributing to explaining its superiority. Moreover, the more expressive parameterization of IQN over C51 is also helpful in interpreting both the acceleration and the improvement in the final performance. Lastly, according to Theorem~\ref{theorem:acceleration}, the access to the agnostic $\kappa$ can serve as a sufficient condition to discriminate whether a specific distributional RL algorithm can accelerate the training in a given environment. However, a precise evaluation of $\kappa$ is tricky, which we leave as valuable future work.

\section{Conclusion, Limitations and Future Work}

In our paper, we answer the question: \textit{how does return distribution in distributional RL help the optimization} from perspectives of the uniform stability and acceleration effect in the optimization.  Our conclusions are made within a new Neural FZI framework that connects the optimization results in supervised learning with practical deep RL algorithms.

\noindent \textbf{Limitations and Future Work.} Our optimization analysis of distributional RL is based on categorical parameterization, and therefore, some optimization properties, such as uniform stability, may not directly apply to other distributional RL families. The alternative analysis on distributional RL algorithms based on Wasserstein distance is also an integral and valuable complement to our conclusions, which we leave as future work.

\section*{Acknowledgements}

Bei Jiang and Linglong Kong were partially supported by grants from the Canada CIFAR AI Chairs program, the Alberta Machine Intelligence Institute (AMII), and Natural Sciences and Engineering Council of Canada (NSERC), and Linglong Kong was also  partially supported by grants from the Canada Research Chair program from NSERC. We also thank all the constructive suggestions and comments from the reviewers.

\bibliographystyle{plain}
\bibliography{TowardsDRL}

\begin{thebibliography}{10}

\bibitem{agarwal2020optimality}
Alekh Agarwal, Sham~M Kakade, Jason~D Lee, and Gaurav Mahajan.
\newblock Optimality and approximation with policy gradient methods in markov
  decision processes.
\newblock In {\em Conference on Learning Theory}, pages 64--66. PMLR, 2020.

\bibitem{ahmed2019understanding}
Zafarali Ahmed, Nicolas Le~Roux, Mohammad Norouzi, and Dale Schuurmans.
\newblock Understanding the impact of entropy on policy optimization.
\newblock In {\em International conference on machine learning}, pages
  151--160. PMLR, 2019.

\bibitem{bellemare2017distributional}
Marc~G Bellemare, Will Dabney, and R{\'e}mi Munos.
\newblock A distributional perspective on reinforcement learning.
\newblock {\em International Conference on Machine Learning (ICML)}, 2017.

\bibitem{bellemare2017cramer}
Marc~G Bellemare, Ivo Danihelka, Will Dabney, Shakir Mohamed, Balaji
  Lakshminarayanan, Stephan Hoyer, and R{\'e}mi Munos.
\newblock The cramer distance as a solution to biased wasserstein gradients.
\newblock {\em arXiv preprint arXiv:1705.10743}, 2017.

\bibitem{bjorck2021towards}
Johan Bjorck, Carla~P Gomes, and Kilian~Q Weinberger.
\newblock Towards deeper deep reinforcement learning.
\newblock {\em Advances in neural information processing systems (NeurIPS)},
  2021.

\bibitem{ceron2021revisiting}
Johan Samir~Obando Ceron and Pablo~Samuel Castro.
\newblock Revisiting rainbow: Promoting more insightful and inclusive deep
  reinforcement learning research.
\newblock In {\em International Conference on Machine Learning}, pages
  1373--1383. PMLR, 2021.

\bibitem{dabney2018implicit}
Will Dabney, Georg Ostrovski, David Silver, and R{\'e}mi Munos.
\newblock Implicit quantile networks for distributional reinforcement learning.
\newblock {\em International Conference on Machine Learning (ICML)}, 2018.

\bibitem{dabney2017distributional}
Will Dabney, Mark Rowland, Marc~G Bellemare, and R{\'e}mi Munos.
\newblock Distributional reinforcement learning with quantile regression.
\newblock {\em Association for the Advancement of Artificial Intelligence
  (AAAI)}, 2018.

\bibitem{fan2020theoretical}
Jianqing Fan, Zhaoran Wang, Yuchen Xie, and Zhuoran Yang.
\newblock A theoretical analysis of deep q-learning.
\newblock In {\em Learning for Dynamics and Control}, pages 486--489. PMLR,
  2020.

\bibitem{gogianu2021spectral}
Florin Gogianu, Tudor Berariu, Mihaela~C Rosca, Claudia Clopath, Lucian
  Busoniu, and Razvan Pascanu.
\newblock Spectral normalisation for deep reinforcement learning: an
  optimisation perspective.
\newblock In {\em International Conference on Machine Learning}, pages
  3734--3744. PMLR, 2021.

\bibitem{gulrajani2017improved}
Ishaan Gulrajani, Faruk Ahmed, Martin Arjovsky, Vincent Dumoulin, and Aaron~C
  Courville.
\newblock Improved training of wasserstein gans.
\newblock {\em Advances in neural information processing systems}, 30, 2017.

\bibitem{haarnoja2017reinforcement}
Tuomas Haarnoja, Haoran Tang, Pieter Abbeel, and Sergey Levine.
\newblock Reinforcement learning with deep energy-based policies.
\newblock In {\em International Conference on Machine Learning}, pages
  1352--1361. PMLR, 2017.

\bibitem{haarnoja2018soft}
Tuomas Haarnoja, Aurick Zhou, Kristian Hartikainen, George Tucker, Sehoon Ha,
  Jie Tan, Vikash Kumar, Henry Zhu, Abhishek Gupta, Pieter Abbeel, et~al.
\newblock Soft actor-critic algorithms and applications.
\newblock {\em arXiv preprint arXiv:1812.05905}, 2018.

\bibitem{hardt2016train}
Moritz Hardt, Ben Recht, and Yoram Singer.
\newblock Train faster, generalize better: Stability of stochastic gradient
  descent.
\newblock In {\em International Conference on Machine Learning}, pages
  1225--1234. PMLR, 2016.

\bibitem{imani2018improving}
Ehsan Imani and Martha White.
\newblock Improving regression performance with distributional losses.
\newblock In {\em International Conference on Machine Learning}, pages
  2157--2166. PMLR, 2018.

\bibitem{li2021functional}
Alexander Li and Deepak Pathak.
\newblock Functional regularization for reinforcement learning via learned
  fourier features.
\newblock {\em Advances in Neural Information Processing Systems}, 34, 2021.

\bibitem{luo2021distributional}
Yudong Luo, Guiliang Liu, Haonan Duan, Oliver Schulte, and Pascal Poupart.
\newblock Distributional reinforcement learning with monotonic splines.
\newblock In {\em International Conference on Learning Representations}, 2021.

\bibitem{lyle2019comparative}
Clare Lyle, Marc~G Bellemare, and Pablo~Samuel Castro.
\newblock A comparative analysis of expected and distributional reinforcement
  learning.
\newblock In {\em Proceedings of the AAAI Conference on Artificial
  Intelligence}, volume~33, pages 4504--4511, 2019.

\bibitem{ma2020dsac}
Xiaoteng Ma, Li~Xia, Zhengyuan Zhou, Jun Yang, and Qianchuan Zhao.
\newblock Dsac: Distributional soft actor critic for risk-sensitive
  reinforcement learning.
\newblock {\em arXiv preprint arXiv:2004.14547}, 2020.

\bibitem{ma2021conservative}
Yecheng~Jason Ma, Dinesh Jayaraman, and Osbert Bastani.
\newblock Conservative offline distributional reinforcement learning.
\newblock {\em arXiv preprint arXiv:2107.06106}, 2021.

\bibitem{mavrin2019distributional}
Borislav Mavrin, Shangtong Zhang, Hengshuai Yao, Linglong Kong, Kaiwen Wu, and
  Yaoliang Yu.
\newblock Distributional reinforcement learning for efficient exploration.
\newblock {\em International Conference on Machine Learning (ICML)}, 2019.

\bibitem{mei2020global}
Jincheng Mei, Chenjun Xiao, Csaba Szepesvari, and Dale Schuurmans.
\newblock On the global convergence rates of softmax policy gradient methods.
\newblock In {\em International Conference on Machine Learning}, pages
  6820--6829. PMLR, 2020.

\bibitem{miyato2018spectral}
Takeru Miyato, Toshiki Kataoka, Masanori Koyama, and Yuichi Yoshida.
\newblock Spectral normalization for generative adversarial networks.
\newblock {\em International Conference on Learning Representations}, 2018.

\bibitem{mnih2015human}
Volodymyr Mnih, Koray Kavukcuoglu, David Silver, Andrei~A Rusu, Joel Veness,
  Marc~G Bellemare, Alex Graves, Martin Riedmiller, Andreas~K Fidjeland, Georg
  Ostrovski, et~al.
\newblock Human-level control through deep reinforcement learning.
\newblock {\em nature}, 518(7540):529--533, 2015.

\bibitem{nguyen2020distributional}
Thanh~Tang Nguyen, Sunil Gupta, and Svetha Venkatesh.
\newblock Distributional reinforcement learning with maximum mean discrepancy.
\newblock {\em Association for the Advancement of Artificial Intelligence
  (AAAI)}, 2020.

\bibitem{riedmiller2005neural}
Martin Riedmiller.
\newblock Neural fitted q iteration--first experiences with a data efficient
  neural reinforcement learning method.
\newblock In {\em European conference on machine learning}, pages 317--328.
  Springer, 2005.

\bibitem{santurkar2018does}
Shibani Santurkar, Dimitris Tsipras, Andrew Ilyas, and Aleksander Madry.
\newblock How does batch normalization help optimization?
\newblock {\em Advances in neural information processing systems}, 31, 2018.

\bibitem{strehl2009reinforcement}
Alexander~L Strehl, Lihong Li, and Michael~L Littman.
\newblock Reinforcement learning in finite mdps: Pac analysis.
\newblock {\em Journal of Machine Learning Research}, 10(11), 2009.

\bibitem{sun2021interpreting}
Ke~Sun, Yingnan Zhao, Yi~Liu, Shi Enze, Wang Yafei, Yan Xiaodong, Bei Jiang,
  and Linglong Kong.
\newblock Interpreting distributional reinforcement learning: A regularization
  perspective.
\newblock {\em arXiv preprint arXiv:2110.03155}, 2021.

\bibitem{sun2022distributional}
Ke~Sun, Yingnan Zhao, Yi~Liu, Bei Jiang, and Linglong Kong.
\newblock Distributional reinforcement learning via sinkhorn iterations.
\newblock {\em arXiv preprint arXiv:2202.00769}, 2022.

\bibitem{sutton2018reinforcement}
Richard~S Sutton and Andrew~G Barto.
\newblock {\em Reinforcement learning: An Introduction}.
\newblock MIT press, 2018.

\bibitem{wasserman2006all}
Larry Wasserman.
\newblock {\em All of nonparametric statistics}.
\newblock Springer Science \& Business Media, 2006.

\bibitem{watkins1992q}
Christopher~JCH Watkins and Peter Dayan.
\newblock Q-learning.
\newblock {\em Machine learning}, 8(3-4):279--292, 1992.

\bibitem{xu2020towards}
Yi~Xu, Yuanhong Xu, Qi~Qian, Hao Li, and Rong Jin.
\newblock Towards understanding label smoothing.
\newblock {\em arXiv preprint arXiv:2006.11653}, 2020.

\bibitem{yang2019fully}
Derek Yang, Li~Zhao, Zichuan Lin, Tao Qin, Jiang Bian, and Tie-Yan Liu.
\newblock Fully parameterized quantile function for distributional
  reinforcement learning.
\newblock {\em Advances in neural information processing systems},
  32:6193--6202, 2019.

\end{thebibliography}

\clearpage
\appendix

\section{Derivation of Categorical Distributional Loss}\label{appendix:histogram}
We show the derivation details of the Categorical distribution loss starting from KL divergence between $p$ and $q_\theta$. $p_i$ is the cumulative probability increment of target distribution $\{Y_i\}_{i\in [n]}$ within the $i$-th bin, and $q_\theta$ corresponds to a (normalized) histogram, and has density values $\frac{f_i^\theta(\mathbf{x}(s))}{w_i}$ per bin. Thus, we have:
\begin{equation}\begin{aligned}
		D_{\text{KL}}\left(p^{s, a}, q^{s, a}_\theta\right)& =\int_{a}^{b} p^{s, a}(y) \log p^{s, a}(y) d y  -\int_{a}^{b} p^{s, a}(y) \log q^{s, a}_\theta(y) d y \\ 
		&\propto -\int_{a}^{b} p^{s, a}(y) \log q^{s, a}_\theta(y) d y\\
		&= -\sum_{i=1}^{k} \int_{z_{i}}^{z_{i}+w_{i}} p^{s, a}(y) \log \frac{f_{i}^\theta(\mathbf{x}(s))}{w_{i}} d y\\
		&=-\sum_{i=1}^{k} \log \frac{f_{i}^\theta(\mathbf{x}(s))}{w_{i}} \underbrace{\left(F^{s, a}\left(z_{i}+w_{i}\right)-F^{s, a}\left(z_{i}\right)\right)}_{p_{i}^{s, a}}\\
		&\propto -\sum_{i=1}^{k} p^{s, a}_{i} \log f_{i}^\theta(\mathbf{x}(s))
\end{aligned}\end{equation}	
where the first $\propto$ results from the fixed target $p^{s, a}$ in the Neural FZI framework. The second equality is based on the categorical parameterization for the density function $q^{s, a}_{\theta}$. The last $\propto$ holds because the width parameter $w_i$ can be ignored for this minimization problem.

\section{Proof of Proposition~\ref{prop:lipschitz}}\label{appendix:lemma_lipschitz}
\begin{proof}
	For the Categorical distributional loss below,
	$$\mathcal{L}_\theta(s, a) = -\sum_{i=1}^{k} p^{s, a}_{i} \log f_{i}^\theta(\mathbf{x}(s)), \ \text{where} \ f_{i}^{\theta}(\mathbf{x}(s))=\frac{\exp \left(\mathbf{x}(s)^{\top}\theta_{i} \right)}{ \sum_{j=1}^{k} \exp \left(\mathbf{x}(s)^{\top} \theta_{j}\right)} $$
	\noindent \textbf{(1) Convexity.} Note that $-\log \frac{\exp \left(\mathbf{x}(s)^{\top}\theta_{i} \right)}{ \sum_{j=1}^{k} \exp \left(\mathbf{x}(s)^{\top} \theta_{j}\right)} = \log \sum_{j=1}^{k} \exp \left(\mathbf{x}(s)^{\top} \theta_{j}\right) -  \mathbf{x}(s)^{\top}\theta_{i}$, the first term is Log-sum-exp, which is convex~(see Convex optimization by Boyd and Vandenberghe), and the second term is affine function. Thus, $\mathcal{L}_\theta(s, a)$ is convex.
	
	\noindent \textbf{(2) $\mathcal{L}_\theta(s, a)$ is $kl$-Lipschitz continuous.} We compute the gradient of the Histogram distributional loss regarding $\theta_i$:
	\begin{equation}\begin{aligned}
			\frac{\partial}{\partial \theta_i} \sum_{j=1}^{k} p^{s, a}_{j} \log f^\theta_{j}(\mathbf{x}(s))
			&=\sum_{j=1}^{k} p^{s, a}_{j} \frac{1}{f^\theta_j(\mathbf{x}(s))} \nabla_{\theta_i} f^\theta_j(\mathbf{x}(s))\\
			&=\sum_{j=1}^{k} p^{s, a}_{j} \frac{1}{f^\theta_j(\mathbf{x}(s))} f^\theta_i(\mathbf{x}(s)) (\delta_{ij}-f^\theta_j(\mathbf{x}(s))) \mathbf{x}(s)\\
			&=\left(p^{s, a}_{i}(1-f^\theta_i(\mathbf{x}(s)))-\sum_{j\neq i}^{k} p^{s, a}_{j} f^\theta_i(\mathbf{x}(s))\right)  \mathbf{x}(s) \\
			&=\left(p^{s, a}_{i}-p^{s, a}_{i} f^\theta_i(\mathbf{x}(s))-(1-p^{s, a}_{i}) f^\theta_i(\mathbf{x}(s))\right)  \mathbf{x}(s)\\
			&=\left( p^{s, a}_{i} - f^\theta_i(\mathbf{x}(s)) \right)  \mathbf{x}(s)
	\end{aligned}\end{equation}	
	where $\delta_{ij}=1$ if $i=j$, otherwise 0. Then, as we have $\Vert \mathbf{x}(s) \Vert \leq l$, we bound the norm of its gradient
	\begin{equation}\begin{aligned}
			\Vert \frac{\partial}{\partial \theta} \sum_{j=1}^{k} p_{j} \log f^\theta_{j}(\mathbf{x}(s)) \Vert
			&\le \sum_{i=1}^{k} \Vert \frac{\partial}{\partial \theta_i} \sum_{j=1}^{k} p_{j} \log f^\theta_{j}(\mathbf{x}(s)) \Vert\\
			&= \sum_{i=1}^k \Vert \left( p^{s, a}_{i} - f^\theta_i(\mathbf{x}(s)) \right)  \mathbf{x}(s) \Vert\\
			&\le \sum_{i=1}^k |p^{s, a}_{i} - f^\theta_i(\mathbf{x}(s))| \Vert \mathbf{x}(s) \Vert \\
			&\leq kl
	\end{aligned}\end{equation}	
	The last equality satisfies because $|p_i - f^\theta_i(\mathbf{x}(s))|$ is less than 1 and even smaller. Therefore, we obtain that $\mathcal{L}_\theta$ is $kl$-Lipschitz.
	
	\noindent \textbf{(3) $\mathcal{L}_\theta$ is $kl^2$-Lipschitz smooth.} A lemma is that $\log(1+\exp(x))$ is $\frac{1}{4}$-smooth as its second-order gradient is bounded by $\frac{1}{4}$, and if $g(w)$ is $\beta$-smooth w.r.t. $w$, then $g(\Braket{x, w})$ is $\beta\Vert x \Vert^2$-smooth. Based on this knowledge, we firstly focus on the 1-dimensional case of the function $\log f^\theta_{j}(z)$, where $f^\theta_{j}(z)=\frac{\exp z_j}{\sum_{i=1}^{k} \exp z_i}$. As we have derived, we know that $\frac{\partial}{\partial \theta_i} \log f^\theta_{j}(z_j) = \delta_{ij}-f^\theta_{i}(z_i)$. Then the second-order gradient is $\frac{\partial^2 \log f^\theta_{j}(z)}{\partial \theta_i \partial \theta_k}=-f^\theta_i(z) (\delta_{ik}-f^\theta_k(z))=f^\theta_i(z)(f^\theta_k(z)-1)$ if $i=k$, otherwise $f^\theta_i(z) f^\theta_k(z)$. Clearly, $|\frac{\partial^2 \log f^\theta_{j}(z)}{\partial \theta_i \partial \theta_k}|\le 1$, which implies that $\log f^\theta_{j}(z)$ is 1-smooth. Thus, $\log f^\theta_{j}(\Braket{x, \theta_i})$ is $\Vert x \Vert^2$-smooth, or $l^2$-smooth. Further, $\sum_{j=1}^{k} p^{s, a}_{j} \log f^\theta_{j}(\mathbf{x}(s))$ is also $l^2$-smooth as we have
	\begin{equation}\begin{aligned}
			&\Vert \nabla_{\theta_{i}} \sum_{j=1}^{k} p^{s, a}_{j} \log f^\mu_{j}(\mathbf{x}(s)) - \nabla_{\theta_{i}} \sum_{j=1}^{k} p^{s, a}_{j} \log f^\nu_{j}(\mathbf{x}(s)) \Vert\\
			&\le \sum_{j=1}^{k} p^{s, a}_{j} \Vert \nabla_{\theta_{i}} \log f^\mu_{j}(\mathbf{x}(s)) - \nabla_{\theta_{i}} \log f^\nu_{j}(\mathbf{x}(s)) \Vert\\
			&\le \sum_{j=1}^{k} p^{s, a}_{j} \cdot l^2 \Vert \mu - \nu \Vert\\
			&=l^2 \Vert \mu - \nu \Vert
	\end{aligned}\end{equation}	
	for each parameter $\mu$ and $\nu$. Therefore, we further have
	\begin{equation}\begin{aligned}
			&\Vert \nabla_{\theta} \sum_{j=1}^{k} p^{s, a}_{j} \log f^\mu_{j}(\mathbf{x}(s)) - \nabla_{\theta} \sum_{j=1}^{k} p^{s, a}_{j} \log f^\nu_{j}(\mathbf{x}(s)) \Vert\\
			&\le \sum_{i=1}^{k}\Vert \nabla_{\theta_{i}} \sum_{j=1}^{k} p^{s, a}_{j} \log f^\mu_{j}(\mathbf{x}(s)) - \nabla_{\theta_{i}} \sum_{j=1}^{k} p^{s, a}_{j} \log f^\nu_{j}(\mathbf{x}(s)) \Vert\\
			&\le \sum_{i=1}^{k} l^2 \Vert \mu - \nu \Vert \\
			& = kl^2 \Vert \mu - \nu \Vert
	\end{aligned}\end{equation}	
	Finally, we conclude that  $\mathcal{L}_\theta(s, a)$ is $kl^2$-smooth.
	
\end{proof}

\section{Proof of Theorem~\ref{theorem:lipschitz}}\label{appendix:lipschitz}

\begin{proof}
	Consider the stochastic gradient descent rule as $G_{\lambda, \mathcal{L}}(\theta)=\theta - \lambda \nabla_{\theta} \mathcal{L}_\theta$. 
	Firstly, we provide two definitions about $\mathcal{L}_\theta$ for the following proof.
	\begin{myDef}
		($\sigma$-bounded) An update rule is $\sigma$-bounded if $\sup_\theta \Vert \theta - \lambda \nabla_{\theta} \mathcal{L}_\theta \Vert \le \sigma$.
	\end{myDef}
	\begin{myDef}
		($\eta$-expansive) An update rule is $\eta$-expansive if $\sup_{v, w} \frac{\Vert G_{\lambda, \mathcal{L}}(v) - G_{\lambda, \mathcal{L}}(w) \Vert}{\Vert u -w \Vert} \le \eta$.
	\end{myDef}
	
	\begin{lemma}\label{appendix:stability_Recursion}
		(Grow Recursion, Lemma 2.5~\cite{hardt2016train}) Fix an arbitrary sequence of updates $G_1, ..., G_T$ and another sequence $G_1^\prime, ..., G_T^\prime$. Let $\theta_0=\theta_0^\prime$ be the starting point and define $\delta_t=\Vert \theta_{i}^\prime - \theta_t  \Vert$, where $\theta_t$ and $\theta_t^\prime$ are defined recursively through
		$$\theta_{t+1}=G_{\lambda, \mathcal{L}}(\theta_{t}), \ \theta_{t+1}^\prime=G^\prime_{\lambda, \mathcal{L}}(\theta_{t}^\prime)$$
		
		Then we have the recurrence relation:
		$$\delta_{t+1} \leq \begin{cases}\eta \delta_{t} & G_{t}=G_{t}^{\prime} \text { is } \eta \text {-expansive } \\ \min (\eta, 1) \delta_{t}+2 \sigma_{t} & G_{t} \text { and } G_{t}^{\prime} \text { are } \sigma \text {-bounded }, G_{t} \text { is } \eta \text { expansive }\end{cases}$$
	\end{lemma}
	
	\begin{lemma}\label{appendix:stability_Lipschitz}
		(Lipschitz Continuity) Assume $\mathcal{L}_\theta$ is $L$-Lipschitz, the gradient update $G_{\lambda, \mathcal{L}}$ is $(\lambda L)$-bounded.
	\end{lemma}
	
	\begin{proof}
		$\Vert \theta -  G_{\lambda, \mathcal{L}}(\theta) \Vert = \Vert \lambda \nabla_{\theta} \mathcal{L}_\theta \Vert \le \lambda L$
	\end{proof}
	
	\begin{lemma}\label{appendix:stability_Smoothness}
		(Lipschitz Smoothness and Convex) Assume $\mathcal{L}_\theta$ is $\beta$-smooth and convex, then for any $\lambda\le\frac{2}{\beta}$, the gradient update $G_{\lambda, \mathcal{L}}$ is 1-expansive.
	\end{lemma}
	\begin{proof}
		Please refer to Lemma 3.7 in \cite{hardt2016train} for the proof.
	\end{proof}
	
	Based on all the results above, we start to prove Theorem~\ref{theorem:lipschitz}. Our proof is largely based on \cite{hardt2016train}, but it is applicable in distributional RL settings and considering desirable properties of histogram distributional loss. According to Proposition~\ref{prop:lipschitz}, we attain that $\mathcal{L}_\theta$ is $kl$-Lipschitz as well as $kl^2$-smooth, and thus based on Lemma~\ref{appendix:stability_Lipschitz} and Lemma~\ref{appendix:stability_Smoothness}, we have	$G_{\lambda, \mathcal{L}}$ is $(\lambda kl)$-bounded, and 1-expansive if $\lambda\le \frac{2}{kl^2}$. In the step $t$, SGD selects samples that are both in $\mathcal{D}$ and $\mathcal{D}^\prime$, with probability $1-\frac{1}{n}$. In this case, $G_t=G_t^\prime$, and thus $\delta_{t+1}\le \delta_t$ as $G_t$ is 1-expansive based on Lemma~\ref{appendix:stability_Recursion}. The other case is that samples selected are different with probability $\frac{1}{n}$, where $\delta_{t+1}\le \delta_t + 2\lambda_t kl$ based on Lemma~\ref{appendix:stability_Recursion}. Thus, if $\lambda_t \le \frac{2}{kl^2}$, for each state $s$ and action $a$, we have:
	\begin{equation}\begin{aligned}
			\mathbb{E}\left|\mathcal{L}_{\theta_T}(s, a) - \mathcal{L}_{\theta_T^\prime}(s, a)\right| &\le kl \mathbb{E}\left[\delta_T\right], \ \text{where} \ \delta_T = \Vert \theta_T - \theta^\prime_T \Vert\\
			&\le kl\left((1-\frac{1}{n}) \mathbb{E}\left[\delta_{T-1}\right] + \frac{1}{n}\mathbb{E}\left[\delta_{T-1}\right]  +  \frac{2\lambda_{T-1} kl}{n}\right)\\
			&= kl \left(\mathbb{E}\left[\delta_{T-1}\right] + \frac{2\lambda_{T-1} kl}{n} \right) \\
			&=kl \left(\mathbb{E}\left[\delta_{0}\right] + \sum_{t=0}^{T-1} \frac{2\lambda_t kl}{n} \right)\\
			&\le \frac{2k^2l^2}{n} \sum_{t=0}^{T-1} \frac{2}{kl^2}\\
			&=\frac{4kT}{n}
	\end{aligned}\end{equation}
	Since this bound holds for all $\mathcal{D}$, $\mathcal{D}^\prime$ and $s, a$, we attain the uniform stability in Definition~\ref{def:stability} for our categorical distributional loss applied in distributional RL.
	
	Define the population risk as:
	$$R\left[\theta\right]=\mathbb{E}_{x}\mathcal{L}_{\theta}(s, a)$$
	and the empirical risk as:
	$$R_S\left[\theta\right]=\frac{1}{n} \sum_{i=1}^{n}\mathcal{L}_{\theta}(s_i, a_i)$$
	According to Theorem 2.2 in \cite{hardt2016train}, if an algorithm $\mathcal{M}$ is $\epsilon_{\text{stab}}$-uniformly stable, then the generalization gap is $\epsilon_{\text{stab}}$-bounded, i.e., 
	$$\left|\mathbb{E}_{S, A}\left[R_{S}[\mathcal{M}(\mathcal{D})]-R[\mathcal{M}(\mathcal{D}^\prime)]\right]\right| \leq \epsilon_{\text{stab}}$$
\end{proof}

\section{Proof of Proposition~\ref{prop:acceleration}}\label{appendix:acceleration_lemma}
\begin{equation}\begin{aligned}
		\mathbb{E}_{(s, a)\sim \rho^\pi}\left[\|\nabla \mathcal{L}_\theta(p^{s, a}, f_\theta^{s, a}))-\nabla G(\theta)\|^{2}\right]\le (1-\epsilon)^2\sigma^{2} + \epsilon^2 \kappa \sigma^{2}.
\end{aligned}\end{equation}
\begin{proof}
	As we know that $p^{s, a}(x)=(1-\epsilon)p_E^{s, a} + \epsilon \mu^{s, a}(x)$ and we use KL divergence in $\mathcal{L}_\theta$, then we have:
	$$\nabla \mathcal{L}_\theta(p^{s, a}, f_\theta^{s, a})=(1-\epsilon)\nabla \mathcal{L}_\theta(p_E^{s, a} , f_\theta^{s, a}) + \epsilon \nabla \mathcal{L}_\theta(\mu^{s, a}, f_\theta^{s, a}) $$
	Therefore, 
	\begin{equation}\begin{aligned}
			&\mathbb{E}_{(s, a)\sim \rho^\pi}\left[\|\nabla \mathcal{L}_\theta(p^{s, a}, f_\theta^{s, a}))-\nabla G(\theta)\|^{2}\right]\\
			&\le \mathbb{E}_{(s, a)\sim \rho^\pi}\left[ (1-\epsilon)^2\|\nabla \mathcal{L}_\theta(p_E^{s, a}, f_\theta^{s, a}))-\nabla G(\theta)\|^{2} + \epsilon^2 \|\nabla \mathcal{L}_\theta(\mu^{s, a}, f_\theta^{s, a}))-\nabla G(\theta)\|^{2} \right]\\
			&= (1-\epsilon)^2\sigma^{2} + \epsilon^2 \kappa \sigma^{2},
	\end{aligned}\end{equation}
	where the first inequality uses the triangle inequality of norm, i.e., $\|(1-\epsilon) \mathbf{a}+\epsilon \mathbf{b}\|^{2} \leq(1-\epsilon)^2\|\mathbf{a}\|^{2}+\epsilon^2\|\mathbf{b}\|^{2}$, and the last equality uses the definition of the variance of $\mathcal{L}_\theta(p_E^{s, a} , f_\theta^{s, a})$ and $\mathcal{L}_\theta(\mu^{s, a}, f_\theta^{s, a})$.
\end{proof}

\section{Proof of Theorem~\ref{theorem:acceleration}}\label{appendix:acceleration_theorm}
\begin{proof}
	\textbf{Classical RL} (1) If we only consider the expectation of $Z^\pi(s, a)$, we use the information $p_E^{s, a}$ to construct the loss function. As $\mathcal{L}_\theta(p_E^{s, a} , q_\theta^{s, a})$ is $kl^2$-smooth, we have
	\begin{equation}\begin{aligned}
			G(\theta_{t+1}) -G(\theta_{t}) &\le \Braket{\nabla G(\theta_{t}), \theta_{t+1}-\theta_{t}} + \frac{kl^2}{2} \Vert \theta_{t+1}-\theta_{t} \Vert^2 \\
			&= -\lambda \Braket{\nabla G(\theta_{t}), \nabla \mathcal{L}_\theta(p_E^{s, a} , f_\theta^{s, a})} + \frac{kl^2\lambda^2}{2} \Vert \nabla \mathcal{L}_\theta(p_E^{s, a} , f_\theta^{s, a}) \Vert^2
	\end{aligned}\end{equation}
	where the inequality is according to the definition of Lipschitz-smoothness, and the last equation is based on the updating rule of $\theta$. Next, we take the expectation on both sides,
	\begin{equation}\begin{aligned}
			&\mathbb{E}\left[G(\theta_{t+1}) -G(\theta_{t})\right]\\ 
			&\le -\lambda \mathbb{E}\left[\Vert \nabla G(\theta_{t}) \Vert^2\right] + \frac{kl^2\lambda^2}{2} \mathbb{E}\left[\Vert \nabla \mathcal{L}_\theta(p_E^{s, a}, f_\theta^{s, a}) - \nabla G(\theta_{t}) +\nabla G(\theta_{t}) \Vert^2\right] \\
			&\le -\lambda \mathbb{E}\left[\Vert \nabla G(\theta_{t}) \Vert^2\right] + \frac{kl^2\lambda^2}{2} \mathbb{E}\left[\Vert \nabla \mathcal{L}_\theta(p_E^{s, a} , f_\theta^{s, a}) - \nabla G(\theta_{t}) \Vert^2\right] + \frac{kl^2\lambda^2}{2} \mathbb{E}\left[\Vert \nabla G(\theta_{t}) \Vert^2\right] \\
			&=  \frac{\lambda (kl^2\lambda -2)}{2} \mathbb{E}\left[\Vert \nabla G(\theta_{t}) \Vert^2\right] + \frac{kl^2\lambda^2}{2} \sigma^2\\
			&\le -\frac{\lambda}{2}  \mathbb{E}\left[\Vert \nabla G(\theta_{t}) \Vert^2\right] + \frac{kl^2\lambda^2}{2} \sigma^2
	\end{aligned}\end{equation}
	where the first two inequalities hold because $\nabla G(\theta)=\mathbb{E}\left[\nabla \mathcal{L}_\theta\right]$ and the last inequality comes from $\lambda \le \frac{1}{kl^2}$. Through the summation, we obtain that 
	$$\mathbb{E}\left[G(\theta_{T}) -G(\theta_{0})\right]\le -\frac{\lambda}{2} \sum_{t=0}^{T-1} \mathbb{E}\left[\Vert \nabla G(\theta_{t}) \Vert^2\right] + \frac{kl^2\lambda^2 T}{2} \sigma^2$$
	We let $\mathbb{E}\left[G(\theta_{T})\right]=0$, we have
	$$\frac{1}{T}\sum_{t=0}^{T-1} \mathbb{E}\left[\Vert \nabla G(\theta_{t}) \Vert^2\right] \le  \frac{2G(\theta_{0})}{\lambda T} + kl^2\lambda \sigma^2$$
	By setting $\lambda\le\frac{\tau^2}{2kl^2\sigma^2}$ (simultaneously $\lambda \le \frac{1}{kl^2}$, i.e., $\lambda \le \frac{1}{kl^2} \min \{1, \frac{\tau^2}{2 \sigma^2}\} $) and $T=\frac{4G(\theta_{0})}{\lambda \tau^2}$, we can have $\frac{1}{T}\sum_{t=0}^{T-1} \mathbb{E}\left[\Vert \nabla G(\theta_{t}) \Vert^2\right] \le \tau^2$, implying that the degenerated loss function based on the expectation $p_E^{s, a}$ can achieve $\tau$-stationary point if the sample complexity $T=O(\frac{1}{\tau^4})$.

	\noindent \textbf{Distributional RL (2).} We are still based on the $kl^2$-smoothness of $\mathcal{L}(p^{s, a}, f^{s, a}_\theta)$.
	\begin{equation}\begin{aligned}
			G(\theta_{t+1}) -G(\theta_{t})&\le \Braket{\nabla G(\theta_{t}), \theta_{t+1}-\theta_{t}} + \frac{kl^2}{2} \Vert \theta_{t+1}-\theta_{t} \Vert^2 \\
			&= -\lambda \Braket{\nabla G(\theta_{t}), \nabla \mathcal{L}_\theta(p^{s, a}, f_\theta^{s, a})} + \frac{kl^2\lambda^2}{2} \Vert \nabla \mathcal{L}_\theta(p^{s, a}, f_\theta^{s, a}) \Vert^2\\
			&= -\frac{\lambda}{2} \Vert \nabla G(\theta_{t}) \Vert^2 + \frac{\lambda}{2} \Vert \nabla G(\theta_{t}) - \nabla \mathcal{L}_\theta(p^{s, a}, f_\theta^{s, a}) \Vert^2 + \frac{\lambda(kl^2 \lambda -1)}{2} \Vert \nabla \mathcal{L}_\theta(p^{s, a}, f_\theta^{s, a}) \Vert^2\\
			&\le -\frac{\lambda}{2} \Vert \nabla G(\theta_{t}) \Vert^2 + \frac{\lambda}{2} \Vert \nabla G(\theta_{t}) - \nabla \mathcal{L}_\theta(p^{s, a}, f_\theta^{s, a}) \Vert^2
	\end{aligned}\end{equation}
	where the second equation is based on $\langle\mathbf{a},-\mathbf{b}\rangle=\frac{1}{2}\left(\|\mathbf{a}-\mathbf{b}\|^{2}-\|\mathbf{a}\|^{2}-\|\mathbf{b}\|^{2}\right)$, and the last inequality is according to $\lambda \le \frac{1}{kl^2}$. After taking the expectation, we have
	\begin{equation}\begin{aligned}
			\mathbb{E}\left[G(\theta_{t+1}) -G(\theta_{t})\right]
			&\le -\frac{\lambda}{2} \mathbb{E}\left[\Vert \nabla G(\theta_{t}) \Vert^2\right] + \frac{\lambda}{2} \mathbb{E}\left[\Vert \nabla G(\theta_{t}) - \nabla \mathcal{L}_\theta(p^{s, a}, f_\theta^{s, a}) \Vert^2 \right] \\
			&\le -\frac{\lambda}{2} \mathbb{E}\left[\Vert \nabla G(\theta_{t}) \Vert^2\right] + \frac{\lambda}{2}\left( (1-\epsilon)^2\sigma^{2} + \epsilon^2 \kappa \sigma^{2} \right)
	\end{aligned}\end{equation}
	where the last inequality is based on Proposition~\ref{prop:acceleration}. We take the summation, and therefore, 
	$$\mathbb{E}\left[G(\theta_{T}) -G(\theta_{0})\right]\le-\frac{\lambda}{2}\sum_{t=0}^{T-1} \mathbb{E}\left[\Vert \nabla G(\theta_{t}) \Vert^2\right] + \frac{T\lambda}{2}\left( (1-\epsilon)^2\sigma^{2} + \epsilon^2 \kappa \sigma^{2} \right)$$
	We let $\mathbb{E}\left[G(\theta_{T})\right]=0$ and $\epsilon=\frac{1}{1+\kappa}$, then,
	\begin{equation}\begin{aligned}
			\frac{1}{T}\sum_{t=0}^{T-1} \mathbb{E}\left[\Vert \nabla G(\theta_{t}) \Vert^2\right] &\le \frac{2 G(\theta_{0})}{\lambda T} + (1-\epsilon)^2\sigma^{2} + \epsilon^2 \kappa \sigma^{2}\\
			&=\frac{2 G(\theta_{0})}{\lambda T} + \frac{2\kappa^2}{(1+\kappa)^2} \sigma^{2} \\
			&\le \frac{2 G(\theta_{0})}{\lambda T} + 2\kappa^2 \sigma^{2}
	\end{aligned}\end{equation}
	If $\kappa \le \frac{\tau}{2\sigma}$ and let $T=\frac{4 G(\theta_0)}{\lambda \tau^2}$, this leads to $\frac{1}{T}\sum_{t=0}^{T-1} \mathbb{E}\left[\Vert \nabla G(\theta_{t}) \Vert^2\right]\le \tau^2$, i.e., $\tau$-stationary point, with the sample complexity as $O(\frac{1}{\tau^2})$. If $\kappa > \frac{\tau}{2\sigma}$, we set $T=\frac{G(\theta_0)}{\lambda \kappa^2 \sigma^2}$. This implies that $\frac{1}{T}\sum_{t=0}^{T-1} \mathbb{E}\left[\Vert \nabla G(\theta_{t}) \Vert^2\right] \le 4\kappa^2\sigma^2$, which can only achieve $2\kappa \sigma$-stationary point. Putting two cases together, we conclude that distributional RL can achieve $\max \{\tau, 2\kappa\sigma\}$-stationary point (since $\tau$ can be pre-given, while $2 \kappa \sigma$ is determined by the environment.)
\end{proof}

\begin{table}[t!]
	\caption{Hyper-parameters Sheet.}
	\label{table:hyperparameters}
	\centering
	\begin{tabular}{lc}\toprule[2pt]
		\specialrule{0pt}{1pt}{1pt}
		Hyperparameter & Value  \\ 
		\hline\specialrule{0pt}{1pt}{1pt}
		\textit{Shared}&~\\
		\quad Policy network learning rate  & 3e-4  \\\specialrule{0pt}{1pt}{1pt}
		\quad (Quantile / Categorical) Value network learning rate & 3e-4  \\\specialrule{0pt}{1pt}{1pt}
		\quad Optimization  & Adam \\\specialrule{0pt}{1pt}{1pt}
		\quad Discount factor  & 0.99 \\\specialrule{0pt}{1pt}{1pt}
		\quad Target smoothing  & 5e-3 \\\specialrule{0pt}{1pt}{1pt}
		\quad Batch size  & 256 \\\specialrule{0pt}{1pt}{1pt}
		\quad Replay buffer size & 1e6 \\\specialrule{0pt}{1pt}{1pt}
		\quad Minimum steps before training  & 1e4 \\\specialrule{0pt}{1pt}{1pt}
		\hline 
		\textit{DAC~(IQN)}&~\\
		\quad Number of quantile fractions ($N$) & 32  \\\specialrule{0pt}{1pt}{1pt}
		\quad Quantile fraction embedding size     & 64   \\\specialrule{0pt}{1pt}{1pt}
		\quad Huber regression threshold    & 1   \\\specialrule{0pt}{1pt}{1pt}
		\hline 
		\textit{DAC~(C51)}&~\\
		\quad Number of Atoms ($k$) & 51  \\\specialrule{0pt}{1pt}{1pt}
		\specialrule{0pt}{1pt}{1pt}\bottomrule[2pt]
	\end{tabular}
	
	\begin{tabular}{lcc}\toprule[2pt]
		\specialrule{0pt}{1pt}{1pt}
		Hyperparameter & $l_k$ for C51 & Max episode lenght  \\ 
		\hline\specialrule{0pt}{1pt}{1pt}
		Walker2d-v2&500 &1000 \\
		Swimmer-v2  & 160&~1000 \\\specialrule{0pt}{1pt}{1pt}
		Reacher-v2  & 500&~1000 \\\specialrule{0pt}{1pt}{1pt}
		Ant-v2 & 500 &1000 \\\specialrule{0pt}{1pt}{1pt}
		HalfCheetah-v2  & 10,000 &1000\\\specialrule{0pt}{1pt}{1pt}
		Humanoid-v2  & 5,000 &1000 \\\specialrule{0pt}{1pt}{1pt}
		HumanoidStandup-v2  & 15,000 &1000 \\\specialrule{0pt}{1pt}{1pt}
		BipedalWalkerHardcore-v2  & 50 &2000 \\\specialrule{0pt}{1pt}{1pt}
		\specialrule{0pt}{1pt}{1pt}\bottomrule[2pt]
	\end{tabular}
	
\end{table}


\section{Implementation Details}\label{appendix:implementation}

Our implementation is directly adapted from the source code in \cite{ma2020dsac}. For DAC~(IQN), we consider the quantile regression for the distribution estimation on the critic loss. Instead of using fixed quantiles in QR-DQN~\cite{dabney2017distributional}, we leverage the quantile fraction generation based on IQN~\cite{dabney2018implicit} that uniformly samples quantile fractions in order to approximate the full quantile function. In particular, we fix the number of quantile fractions as $N$ and keep them ascending. Besides, we adapt the sampling as $\tau_0=0, \tau_i=\epsilon_i/\sum_{i=0}^{N-1}$, where $\epsilon_i \in U[0, 1], i=1,...,N$.

\subsection{Hyper-parameters and Network structure}

We adopt the same hyper-parameters listed in Table~\ref{table:hyperparameters} and network structure as in the original distributional SAC paper~\cite{ma2020dsac}.

\subsection{Best $l_k$ for DAC~(C51)}

As suggested in Table~\ref{table:hyperparameters}, after a line search for the hyperparameter tuning, we select $l_k$ as 500, 10,000, 15,000, 160, 50, 5,000, 500, 500 for ant, halfcheetah, humanoidstand, swimmer, bipedalwalkerhardcore, humanoid, walker2d and reacher, respectively.

%

\section{Equivalence between the loss function in Theorem~\ref{theorem:acceleration} and mean squared loss in Neural FQI}\label{appendix:equivalence}

\begin{prop} (Equivalence between  the first term in Decomposed Neural FZI and Neural FQI)
	In Neural FZI, if the function class $\{Z_\theta: \theta \in \Theta\}$ is sufficiently large such that it contains the target $\{Y_i\}_{i=1}^n$. As $\Delta \rightarrow 0$, minimizing \textbf{the first term} in implies
	\begin{equation}\begin{aligned}\label{eq:decomposition_firstterm}
			P(Z^{k+1}_\theta(s, a)=\mathcal{T}^{\text{opt}} Q^k_{\theta^*}(s, a) ) = 1, \quad  \forall k.
	\end{aligned}\end{equation}	
\end{prop}

\begin{proof}
	
	Firstly, we define the distributional Bellman optimality operator $\mathfrak{T}^{\text{opt}}$ as follows:
	
	\begin{equation}\begin{aligned}\label{eq:distributionaloptimialityoperator}
			&\mathfrak{T}^{\text{opt}} Z(s, a) \stackrel{D}{=} R(s, a)+\gamma Z\left(S^{\prime}, a^*\right), S^{\prime} \sim P(\cdot \mid s, a), \quad a^*=\underset{a^{\prime}}{\operatorname{argmax}} \mathbb{E}\left[Z\left(S^{\prime}, a^{\prime}\right)\right]
	\end{aligned}\end{equation}
	
	If $\{Z_\theta: \theta \in \Theta\}$ is sufficiently large enough such that it contains $\mathfrak{T}^{\text{opt}}Z_{\theta^*}$, then optimizing Neural FZI in Eq.~\ref{eq:Neural_Z_fitting} leads to $Z_\theta^{k+1}=\mathfrak{T}^{\text{opt}}Z_{\theta^*}$.
	
	We apply the action-value density function decomposition on the target histogram function $\widehat{p}^{s, a}(x)$. Consider the parameterized histogram density function $h_\theta$ and denote $h^E_\theta / \Delta$ as the bin height in the bin $\Delta_E$, under the KL divergence between the first histogram function $ \mathds{1}(x\in \Delta_E)$ with $h_\theta(x)$, the objective function is simplified as
	\begin{equation}\begin{aligned}
			D_{\text{KL}}(\mathds{1}(x\in \Delta_E)/\Delta, h_\theta(x))&\propto - \int_{x \in \Delta_E} \frac{1}{\Delta} \log \frac{h_\theta^E}{\Delta} dx \propto - \log h_\theta^E
	\end{aligned}\end{equation}
	Since $\{Z_\theta: \theta \in \Theta\}$ is sufficiently large enough, the KL minimizer would be $\widehat{h}_\theta = \mathds{1}(x \in \Delta_E) / \Delta$ in expectation. Then, $\arg\min_{h_\theta} \lim_{\Delta \rightarrow 0} D_{\text{KL}}(\mathds{1}(x\in \Delta_E)/\Delta, h_\theta(x)) = \delta_{\mathbb{E}\left[Z^{\text{target}}(s,a)\right]}$, where $\delta_{\mathbb{E}\left[Z^{\text{target}}(s,a)\right]}$ is a Dirac Delta function centered at $\mathbb{E}\left[Z^{\text{target}}(s,a)\right]$ and can be viewed as a generalized probability density function. This also applies from $h_\theta$ to $Z_\theta$. In Neural FZI, we have $Z^{\text{target}}=\mathfrak{T}^{\text{opt}}Z_{\theta^*}$. According to the definition of the Dirac function, as $\Delta \rightarrow 0$, we attain
	\begin{equation}\begin{aligned}
			P(Z^{k+1}_\theta(s, a)=\mathbb{E}\left[\mathfrak{T}^{\text{opt}}Z^k_{\theta^*}(s, a)\right]) = 1
	\end{aligned}\end{equation}
	Due to the linearity of expectation analyzed in Lemma 4 of \cite{bellemare2017distributional}, we have
	\begin{equation}\begin{aligned}
			\mathbb{E}\left[\mathfrak{T}^{\text{opt}}Z^k_{\theta^*}(s, a)\right] &= \mathfrak{T}^{\text{opt}}  \mathbb{E}\left[Z^k_{\theta^*}(s, a)\right]= \mathcal{T}^{\text{opt}} Q^k_{\theta^*}(s, a) 
	\end{aligned}\end{equation}	
	Finally, we obtain:
	\begin{equation}\begin{aligned}
			P(Z^{k+1}_\theta(s, a)=\mathcal{T}^{\text{opt}} Q^k_{\theta^*}(s, a) ) = 1 \quad \text{as} \ \ \Delta \rightarrow 0
	\end{aligned}\end{equation}	
\end{proof}

\section{Experimental Results on Acceleration Effects of Distributional RL}\label{appendix:exp_acceleration_kappa}

\paragraph{Same Architecture.} For a fair comparison, we keep the same DAC network architecture and evaluate the gradient norms of DAC~(C51) and a variant of AC, which is optimized based on the expectation of the represented value distribution within the DAC implementation framework. Figure~\ref{fig:acceleration_kappa} suggests DAC~(C51) still enjoys smaller gradient norms than AC in this fair comparison setting.

\begin{figure}[htbp]
	\centering
	\includegraphics[width=1.0\textwidth,trim=0 0 0 0,clip]{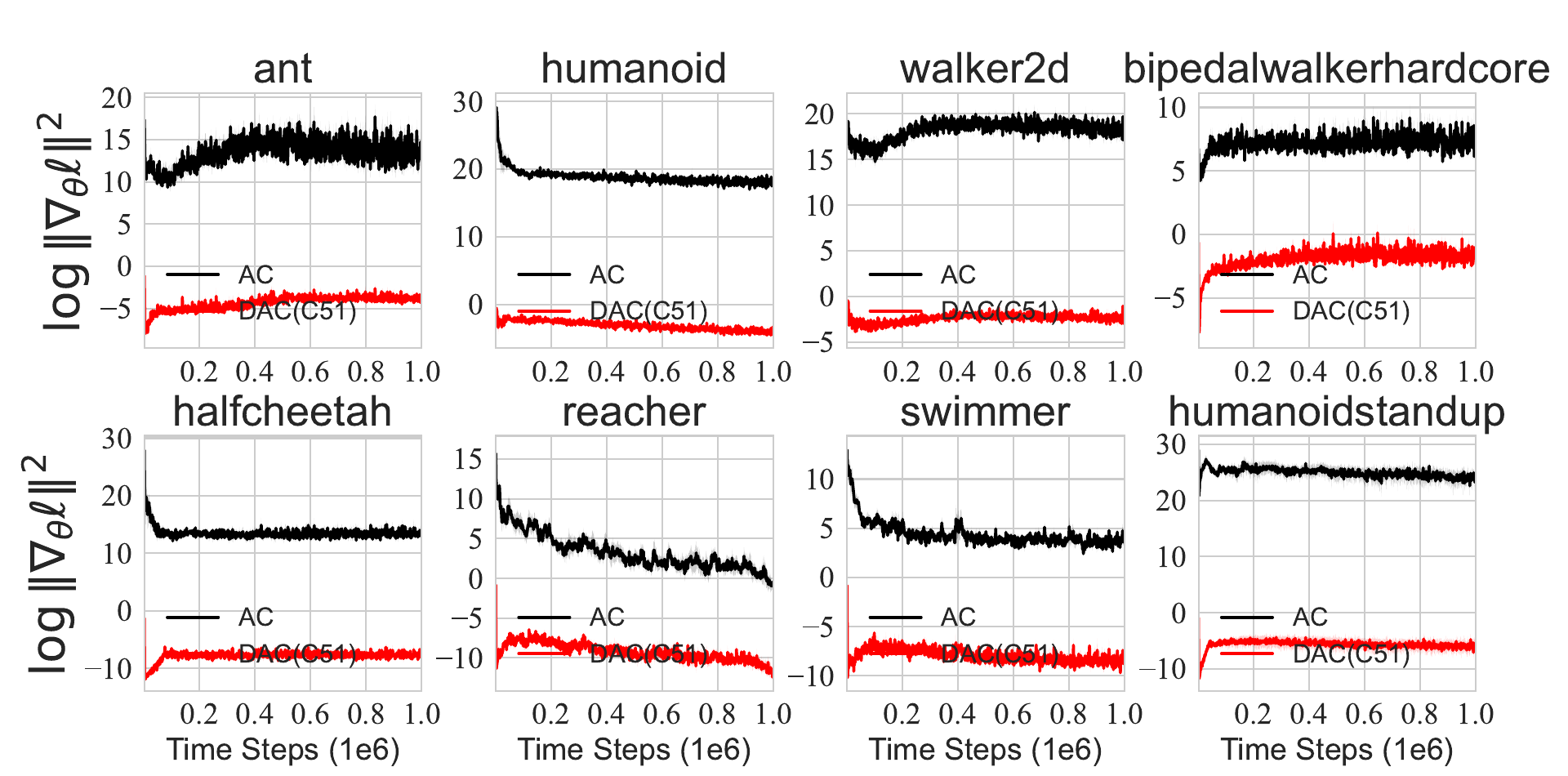}
	\caption{The critic gradient norms in the logarithmic scale during the training of AC and DAC~(C51) over five seeds on three MuJoCo games. \textbf{We keep the same DAC network architecture and evaluate based on the expectation of the represented value distribution}.}
	\label{fig:acceleration_kappa}
\end{figure}

\paragraph{Results under Return Density Decomposition}

We also provide gradient norms of both expectation and distribution based on the Return Density Function decomposition in Eq.~\ref{eq:decomposition}. Similar results can still be observed in Figure~\ref{fig:acceleration_kappacom}.

\begin{figure}[htbp]
	\centering
	\includegraphics[width=1.0\textwidth,trim=0 0 0 0,clip]{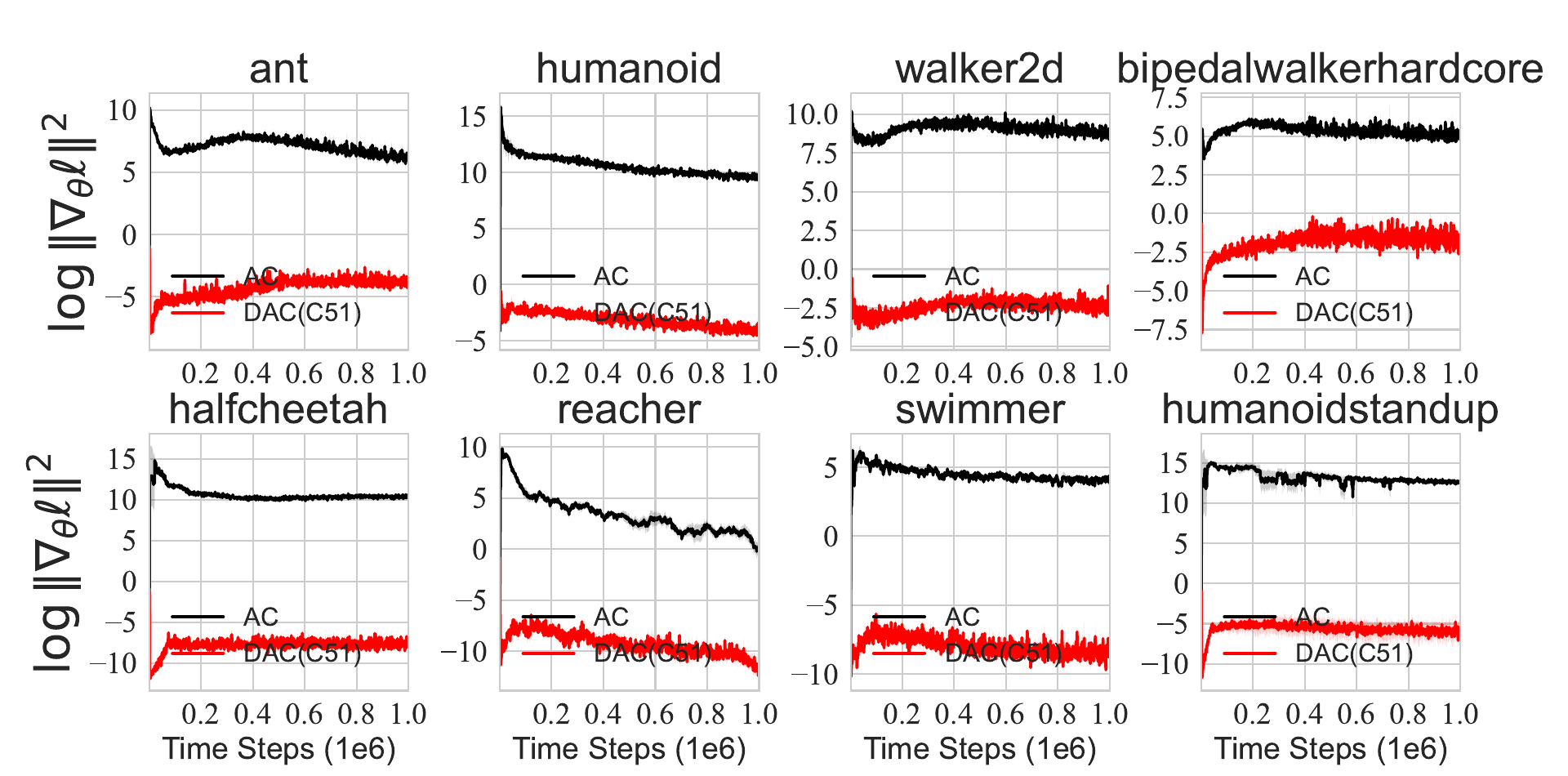}
	\caption{The critic gradient norms in the logarithmic scale during the training of AC and DAC~(C51) over five seeds on three MuJoCo games. \textbf{Results of AC is the expectation part calculated via the Return Density Function Decomposition}.}
	\label{fig:acceleration_kappacom}
\end{figure}

\end{document}